\documentclass[jounal]{IEEEtran}
\IEEEoverridecommandlockouts

\usepackage{graphics} 
\usepackage{epsfig} 
\usepackage{times} 
\usepackage{amsmath} 
\usepackage{amssymb}  
\usepackage{amsfonts,amssymb,amsmath,latexsym}
\usepackage{multirow,multicol}
\usepackage[ruled,linesnumbered]{algorithm2e}
\usepackage{algpseudocode}
\usepackage{hyperref}
\usepackage{epstopdf} 
\usepackage{subfigure} 
\usepackage{stfloats} 
\usepackage{bm}
\usepackage{color}
\usepackage{amsthm}
\usepackage[square,comma,numbers,sort&compress]{natbib}
\usepackage{array}
\usepackage{threeparttable}  
\usepackage{pdfpages}

\usepackage{soul}
\soulregister\citep7
\soulregister\cite7
\soulregister\citet7
\soulregister\ref7

\usepackage{booktabs}

\newcommand{\tabincell}[2]{\begin{tabular}{@{}#1@{}}#2\end{tabular}}

\newtheorem{theorem}{\indent Theorem}
\newtheorem{definition}{\indent Definition}
\theoremstyle{remark}
\newtheorem{remark}{\indent Remark}

\title{Rule-Based Reinforcement Learning for Efficient Robot Navigation with Space Reduction}
\author{Yuanyang~Zhu,~\IEEEmembership{Student Member,~IEEE}, Zhi~Wang,~\IEEEmembership{Member,~IEEE}, Chunlin~Chen,~\IEEEmembership{Senior Member,~IEEE}, Daoyi~Dong,~\IEEEmembership{Senior Member,~IEEE}
	\thanks{Accepted by \textit{IEEE/ASME Transactions on Mechatronics}, 2021, DOI: 10.1109/TMECH.2021.3072675.
	This work was supported in part by the National Natural Science Foundation of China (Nos. 71732003, 62073160 \& 62006111), the National Key Research and Development Program of China (No. 2018AAA0101100), the Synergistic Innovation Center of  Jiangsu Modern Agricultural Equipment and Technology (No. 4091600002), and the Australian Research Councils's Discovery Projects funding scheme under Project DP190101566. (\textit{Corresponding author: Zhi Wang}.)}
	
	\thanks{Y. Zhu and Z. Wang are with the Department of Control and Systems Engineering, School of Management and Engineering, Nanjing University, Nanjing 210093, China (e-mail: yuanyang@smail.nju.edu.cn; zhiwang@nju.edu.cn).}
	
	\thanks{C. Chen is with the Department of Control and Systems Engineering, School of Management and Engineering, Nanjing University, Nanjing 210093, China and with the Synergistic Innovation Center of Jiangsu Modern Agricultural Equipment and  Technology, Zhenjiang 212013, China (e-mail: clchen@nju.edu.cn).}
	
	\thanks{D. Dong is with the School of Engineering and Information Technology, University of New South Wales, Canberra, ACT 2600, Australia (email: daoyidong@gmail.com).}}

\begin{document}
\maketitle
\begin{abstract}
    For real-world deployments, it is critical to allow robots to navigate in complex environments autonomously.
	Traditional methods usually maintain an internal map of the environment, and then design several simple rules, in conjunction with a localization and planning approach, to navigate through the internal map.
	These approaches often involve a variety of assumptions and prior knowledge.
	In contrast, recent reinforcement learning (RL) methods can provide a model-free, self-learning mechanism as the robot interacts with an initially unknown environment, but are expensive to deploy in real-world scenarios due to inefficient exploration.
	In this paper, we focus on efficient navigation with the RL technique and combine the advantages of these two kinds of methods into a rule-based RL (RuRL) algorithm for reducing the sample complexity and cost of time.
	First, we use the rule of wall-following to generate a closed-loop trajectory. 
	Second, we employ a reduction rule to shrink the trajectory, which in turn effectively reduces the redundant exploration space. 
	Besides, we give the detailed theoretical guarantee that the optimal navigation path is still in the reduced space.
	Third, in the reduced space, we utilize the Pledge rule to guide the exploration strategy for accelerating the RL process at the early stage.
	Experiments conducted on real robot navigation problems in hex-grid environments demonstrate that RuRL can achieve improved navigation performance.
\end{abstract}


\section{Introduction}\label{Sec1}
\IEEEPARstart{A}{utonomous} mobile robots are becoming ubiquitous in academia, industrial applications, and our daily life~\citep{faust2018prm,chen2008hybrid}.  
As one of the fundamental topics in the research of mobile robots, robot navigation can be seen as a sequence of translations and rotations for finding the destination, while avoiding obstacles in the environment~\citep{wang2019incremental}.
Enabling mobile robots to perceive and navigate through the surroundings is essential for their successful deployment in real-world scenarios~\citep{7953692}.

Many algorithms have been proposed for path planning and optimization in robot navigation~\cite{holte1996hierarchical}.
Traditional methods usually maintain an internal map of the environment and design simple rules to navigate through the internal map. 
Fuzzy logic methods use fuzzy rules like IF-THEN to make robot navigation decisions~\citep{lilly2007evolution}.
Neuro-fuzzy techniques combine neural networks with fuzzy rules to improve the tracking performance under uncertain physical interaction and external dynamics~\citep{li2020neural,li2020parallel}.
However, it is challenging for human experts to choose the most appropriate rules and membership functions~\citep{fu2018input}.
Another line is to use robotic navigation technologies inspired by biological behavior rules, such as genetic algorithms~\citep{hu2004knowledge}, particle swarm optimization~\citep{hong2019real}, and ant colony optimization (ACO)~\citep{englot2011multi}. 
Owing to the fact that these rules need to know the prior environment model and consider extensive possible situations in advance for mimicking the cognitive process of human experts to solve decision-making problems, rule-based methods tend to converge early to sub-optimal policies~\citep{naser2016ruled}.
Hence, the real-time performance in these methods may not be sufficient to meet the requirements of planning speed and accuracy in path planning tasks~\citep{ju2008evolutionary}.

Recent reinforcement learning (RL) methods offer considerable potentials for mobile robot systems~\citep{dong2012robust}.
RL techniques are obliged to the idea of Markov decision processes (MDPs) that aim to directly solve the optimal sequential decision-making problem of learning from interaction to achieve the goal~\citep{sutton2018reinforcement}. 
By observing the results of navigation decisions in the real world, mobile robots can directly learn from trial-and-error experience, continuously improving their proficiency and adapting to unknown environments~\citep{chen2008hybrid}.
In recent years, RL has been widely investigated in robot navigation domains due to its self-learning and online learning capabilities~\citep{wang2021lifelong}.
However, interacting with the real world can be expensive due to practical constraints such as power usage and human supervision~\citep{pan2018efficient}.
Model-free RL systems are capable of solving complex MDPs in a variety of complex domains, but usually at the cost of a large amount of agent-environment experience due to their limited sample efficiency~\citep{chen2008hybrid}.

Rule-based machine learning (RBML) that combines rules with learning-based methods is a promising direction for utilizing the experts' knowledge to improve the learning performance.
RBML usually covers any machine learning method that identifies, learns, or evolves ``rules" to store, manipulate or apply by the learning system~\citep{bassel2011functional}.
RBML has been widely studied in a variety of fields, such as learning classifier systems, association rule mining, and artificial immune systems, which successfully combines the efficiency of rules and the autonomy of machine learning to complete complex tasks in the real world~\citep{urbanowicz2009learning}.
In the RL community, some researchers employ rules to improve the learning performance in dynamic simulation systems~\citep{munoz2012fuzzy} and robot manipulators' navigation tasks~\citep{althoefer2001reinforcement}.
Nevertheless, few practical implementations of rule-based RL (RuRL) methods have been systematically investigated for robot navigation.

On one hand, traditional rule-based methods generally rely on the environment model and expert knowledge to solve robot navigation tasks, and tend to converge early to sub-optimal policies.
On the other hand, recent RL methods can learn the global optimal policies in a model-free way as the robot interacts with an initially unknown environment, but are expensive to deploy in the real world due to inefficient exploration~\citep{dong2012robust,wang2020incremental}.
In this paper, considering the abilities of the rule-based techniques for logic reasoning and RL methods for solving complex MDPs, we combine the advantages of these two methods into RuRL methods for efficient robot navigation tasks in hex-grid environments.
\footnote{Compared to the triangular and square grids, the hexagonal grid has six equidistant action directions with higher degree of freedom, and may better conform to uneven ground under the same unit area.
The formed trajectory may be smoother in hex-grid maps~\citep{rothman2004lattice}.
Moreover, biological investigations~\citep{hafting2005microstructure} also suggest that neural cognition of spatial navigation is hexagonal.
Hence, we rasterize environments into hexagonal grids here.}

In summary, our main contributions are threefold:
\begin{itemize}
  \item We design the rule of wall-following to obtain a closed-loop trajectory from the starting point to the goal.
We maintain the main angle of view tracking and the priority of action selection strategy to ensure that the mobile robot walks along the left and the right walls, respectively.
\item We use the reduction rule to shrink the trajectory, which effectively reduces the exploration space.
We traverse the obtained trajectory to determine whether there is a shorter path between two given states than the path on the trajectory. 
Besides, we provide the theoretical guarantee that the optimal path is still in the reduced space.
\item We employ the Pledge rule to guide the mobile robot to explore more efficiently at the early learning stage.
Experimental results demonstrate the effectiveness and improved performance of RuRL for robot navigation.
\end{itemize}
These rules reduce the redundant space and accelerate the early exploration to provide coarse-grained learning, which is followed by fine-grained learning using the RL methods with improved efficiency.  
The efficiency is verified by experiments on real-world mobile robot systems in hex-grid environments. 

The rest of this paper is organized as follows. 
Section \ref{Sec2} introduces basic concepts of RL and related work about the efficient exploration methods for RL.
Section \ref{Sec3} presents the integrated RuRL algorithm, including the rule of wall-following, the reduction rule, and the Pledge rule.
The experimental results are discussed in Section \ref{Sec4}, and concluding remarks are drawn in Section \ref{Sec5}.

\section{Preliminaries and related work}\label{Sec2}
\subsection{Reinforcement Learning}\label{Sec2.1}

RL is originated from the idea of MDPs in the field of optimal sequential decision-making problems.
A finite MDP is a tuple of $\langle S,A,T,R,\gamma\rangle$, where $S$ is the set of states, $A$ is the set of actions, $T:S\times A\times S \to [0,1]$ is the state transition probability upon taking action $a$ in state $s$, ${R:S\times A}\to\mathbb{R}$ is the reward function, and $\gamma \in[0,1)$ is the discount factor.
A policy, ${\pi:S\times A}\to [0,1]$, defines how a learner interacts with the environment by mapping
perceived environmental states to actions, and $\sum_{a\in {A}}\pi(a|s)=1,\forall s\in {S}$.
The success of an agent depends on how to maximize the total rewards in the long run when acting under some policy $\pi$.
The goal of RL is to find an optimal policy $\pi^*=\arg\max_{\pi} J(\pi)$ that maximizes the expected long-term return from the distribution
\begin{equation}
J(\pi) = \mathbb{E}_{\tau\sim \pi(\tau)}[r(\tau)] =  \mathbb{E}_{\tau\sim \pi(\tau)}\left[\sum\nolimits_{t=0}^{\infty}\gamma^tr_t\right],
\label{Gpi}
\end{equation}
where $\tau=(s_0, a_0, s_1, a_1, ...)$ is the learning episode, $\pi(\tau)=p(s_0)\Pi_{t=0}^{\infty}\pi(a_t|s_t)p(s_{t+1}|s_t,a_t)$, $r_t$ is the immediate reward received on the transition from $s_t$ to $s_{t+1}$ under action $a_t$.

The expected total reward when executing action $a$ in state $s$ is related to the optimal action-value function $Q^*(s,a)$ as
\begin{equation}
Q^*(s,a) = \max_{\pi}\mathbb{E}_{\pi}\left[\sum\nolimits_{t=0}^{\infty}\gamma^{t}r_{t}\vert s_t=s,a_t=a\right],
\end{equation}
and satisfies the Bellman optimality equation~\citep{bellman2013dynamic}:
\begin{equation}
Q^*(s,a)=\mathbb{E}_{s'}\left[r+\gamma\max_{a'}Q^*(s',a')|s,a\right].
\end{equation}
For a discrete state-action space, the popular Q-learning ~\citep{sutton2018reinforcement} updates the action-value function with a learning rate $\alpha$ as
\begin{equation}
Q(s, a)  \leftarrow Q(s, a) + \alpha\left[r+ \gamma \max _{a'} Q(s', a')- Q(s, a)\right].
\label{q-update}
\end{equation}
In learning, the transition tuples ($s,a,r,s'$) are generated by a behavior policy that can be any exploration policy in principle.
Using these transitions, the Q-function is iteratively updated until converging to the optimal value function $Q^*(s,a)$, and the optimal policy is naturally derived as $\pi^{*}(a|s)=\arg\max_{a}Q^*(s,a)$.
To ensure that $\pi$ converges to the optimal policy, we can select the behavior policy to be $\varepsilon$-soft (e.g., the $\varepsilon$-greedy policy) so that each state-action pair will be visited for an infinite number of times theoretically.
More details about Q-learning can be found in~\citep{sutton2018reinforcement}.

\subsection{Related Work}\label{Sec2.2}
Learning control methods have been widely investigated for solving complex robot control problems.
Learning control signifies that the control system develops representations of the system's mathematical model and derives optimal control laws.
Many learning control techniques have been applied to control systems, mainly consisting of iterative feedback tuning, control loop learning, and machine learning methods. 
For example, to cope with a class of second-order servo systems, the iterative feedback tunning approach employed numerical iterative optimization techniques based on Hessians of output errors and control signal data from the closed-loop system~\citep{preitl2007iterative}. 
A control loop learning method applied two PID control loops to a parallel manipulator aiming to identify models of the robot with a manual approach~\citep{carbone2013design}. 
To reduce traffic fatalities, supervised learning was employed to classify different movement events of pedestrians~\citep{ahmed2019machine}.

While RL has confirmed its ability to learn control strategies for various tasks, e.g., robot navigation, its performance in terms of sample efficiency is still a major challenge in complex applications. 
Many exploration techniques have been investigated to improve the learning efficiency of RL, which can be categorized into undirected and directed exploration strategies according to whether they utilize exploration-specific knowledge of the learning process itself.
Undirected exploration strategies explore the environment based on randomness, such as $\varepsilon$-greedy~\citep{sutton2018reinforcement}, Boltzmann-distributed~\citep{sutton2018reinforcement} and Gaussian noise methods~\citep{osband2014generalization}.
Gaussian noise methods apply Gaussian noise to the action space or parameter space to generate noisy actions for provably efficient exploration~\citep{osband2014generalization}.
Without utilizing any internal information of the learning process, these exploration strategies bring exponential regret in discrete MDPs and are limited to linear function approximations~\citep{houthooft2016vime}.

Directed exploration strategies utilize the previous history of the learning process and influence the portion of the  environment explored in the future, including count-based~\citep{bellemare2016unifying}, curiosity-driven~\citep{houthooft2016vime,li2020random}, and upper confidence bounds (UCB) exploration~\citep{chen2017ucb}.
For tabular-based RL, count-based exploration strategies give an extra exploration bonus to frequently visited states~\citep{bellemare2016unifying}.
In large or continuous state spaces, the pseudo-counts methods employ the density model to obtain pseudo counts from the raw pixels and convert them into an exploration bonus~\citep{bellemare2016unifying}.
Neural density model methods utilize the PixelCNN to provide an exploration bonus derived from an online density model~\citep{ostrovski2017count}.
In large state-action space where the states are rarely visited multiple times, count-based methods are easy to obtain sub-optimal policies owing to paying more attention to visited states only~\citep{tang2017exploration}.
By comparison, we use rules to efficiently reduce the redundant exploration space in complex environments, and theoretically prove that the optimal policy is still in reduced space.

In contrast to count-based methods, curiosity-driven exploration uses a mechanism for generating intrinsic reward signals towards seeking out state-action regions that the agent rarely explores~\citep{houthooft2016vime}.
Intrinsically motivated goal exploration processes explore more states which are fewer experienced in disentangled goal space, and lead to more efficient exploration than the entangled one~\citep{laversanne2018curiosity}.
The curious object-based search agent method~\citep{watters2019cobra} learns representations of the environment without extrinsic reward during the task-free exploration phase, and can be subsequently applied well in other tasks. 
Since the exploration bonus is not dependent on the reward, the main disadvantage is that the exploration may concentrate on irrelevant aspects of the environment~\citep{chen2017ucb}.
In contrast, our method explores efficiently in smaller space, paying less attention to the irrelevant aspects of the environment.

Compared to these methods, the UCB methods design a mechanism for computing the upper confidence bounds of Q-values, and add decaying exploration bonuses to frequently visited states for optimistic exploration~\citep{jin2018q}.
The discounted UCB1-tuned method considers the variance of reward, and uses the weighted variance of the Q-values to reduce exploration regrets~\citep{saito2014discounted}.
UCB exploration via the Q-ensemble method computes the empirical mean and standard deviation of an ensemble of Q-value estimates to reduce the uncertainty for exploration~\citep{chen2017ucb}.
UCB Bernstein approach achieves lower regrets by deriving a coarse bound on the empirical variance of value functions~\citep{jin2018q}.
While these methods are more efficient with UCB exploration strategies, the agent may not efficiently learn in large  state space since the regret scales linearly in the dimension of state space~\citep{jin2019provably}.
Here, the improved learning efficiency obtained by our method is more pronounced in a multi-room environment, which is supposed to benefit from the rule for efficiently reducing the redundant exploration space.

\section{Rule-based RL (RuRL) for Navigation}\label{Sec3}
In this section, we present the framework of RuRL with specific implementations of the rules for generating the closed-loop trajectory, reducing the exploration space, and guiding the early exploration strategy. Then, we give the integrated RuRL algorithm using these implementations.

\subsection{Framework}\label{Sec3.1}
We focus on the mobile robot navigation problem using RL in the hexagonal map environments.
The simultaneous localization and mapping (SLAM) system running on the robot operating system (ROS) platform~\cite{hess2016real} is employed to construct the map of the unknown environment.
The environment perception tasks can be tackled by utilizing the scan matching technique from a fusion between the lidar and ultrasonic sensors.
An inertial measurement unit (IMU) sensor is used to estimate the rotational angle for improving the accuracy of the scan matching method.
The map is created by Cartographer algorithms~\cite{hess2016real}, and rasterized into hexagonal grids using the double-width coordinate system~\citep{birch2007rectangular}.
\footnote{Based on two orthogonal axes, the double-width coordinate system steps to the right by 1 unit, and steps to the below by 2 units~\citep{hoogeboom2018hexaconv}.}

One fundamental problem faced by RL for robot navigation is that the state space can be vast, and consequently, there may be a long delay before the reward is received. 
By applying a machine learning method to automatically discover useful rules, RBML allocates the learning mode in the cooperation rules, making the algorithm effective and flexible.
The individually interpretable rules are clearly defined and applied to challenging tasks that are time-consuming or difficult for data-driven methods.
These rules can model domain-specific knowledge and help speed up the RL process.
Hence, we design several rules to facilitate the navigation performance using RL under a hex-grid map environment.
First, a closed-loop trajectory is generated using two specific trajectories, i.e., $\mathcal{T}_{l}$ and $\mathcal{T}_{r}$, obtained by the left- and right-hand rules, respectively.
Second, based on the closed-loop trajectory, the space reduction rule is employed to form a reduced closed-loop state-action space for reducing redundant exploration space using a proper optimization step $K$.
Finally, in the reduced space, when the number of learning steps exceeds a threshold value $E$ without reaching the goal, the Pledge rule is utilized to guide the exploration strategy for finding the goal with fewer steps at the early learning stage.
In the paper, we adopt the widely used Q-learning as the basic implementation algorithm.
The framework of RuRL is illustrated by the flow diagram as shown in Fig.~\ref{fig:framework}, and the rules and the integrated algorithm are presented in detail in the following subsections.

\begin{figure}[tb]
	\centering
	\includegraphics[width=0.9\columnwidth]{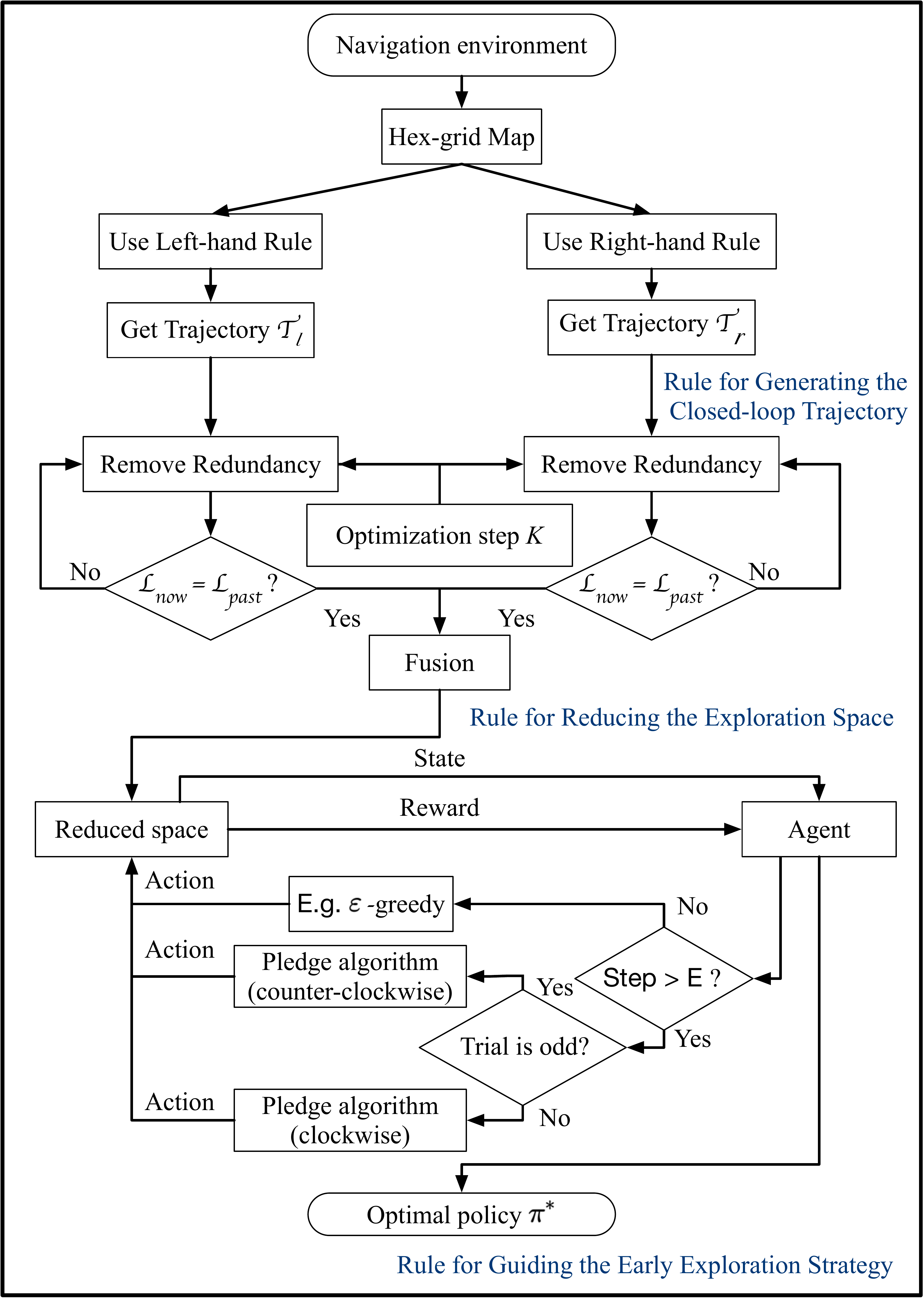}
	\caption{The flow diagram of rule-based RL for robot navigation.}
	\label{fig:framework}
\end{figure}

\subsection{Rule for Generating the Closed-Loop Trajectory}\label{Sec3.2}

After the SLAM system obtains the environment map, we rasterize it into hexagonal grids as shown in Fig.~\ref{fig:hexgrids}.
In hexagonal grids, each state has more action options than in the square grids, making the planned path smoother.
We adopt the widely used double-width coordinate system to calibrate the hexagonal environment.
The origin of the coordinate system is at the top left corner, and the adjacent hexagonal grids along the horizontal and the vertical coordinate axis differ by one and two units, respectively.
Let $l$ and $w$ denote the length and width of the map.
Then, the numbers of columns and rows of the hex-grid map, $m$ and $n$, are calculated as
\begin{equation}
	\left\{\begin{array}{l}{(n-1) * \frac{\sqrt{3}}{2} * a=w}, \\ {(m+1) * \frac{3}{2} * a-a=l},\end{array}\right.
	\label{equ:col}
\end{equation}
where $a$ is the hexagonal edge length.
Given a state that corresponds to a hexagonal grid, there are six available actions.
Correspondingly, from the first perspective of the mobile robot, the six available actions are: front ($F$), right front ($RF$), right rear ($RR$), rear ($R$), left rear ($LR$), and left front ($LF$).

\begin{figure}[tb]
	\centering
	\includegraphics[width=0.7\columnwidth]{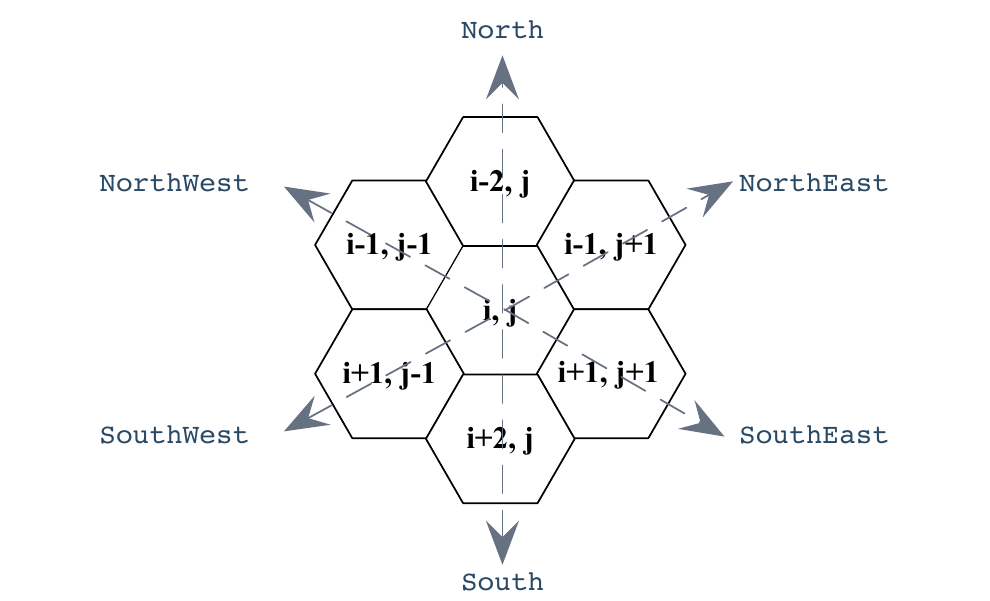}
	\caption{The coordinate system of the hexagonal map.}
	\label{fig:hexgrids}
\end{figure}

Since the direction of the main angle of view changes when the robot moves, we need to record the previous action $a_{t-1}$ to determine the perspective of the robot at the current time step $t$.
For example, when $a_{t-1}=RF$, the main angle of view is the direction of the right front from the perspective of the mobile robot.
Now, we aim to use the left- and the right-hand rules to generate a closed-loop trajectory along the wall. 
The right-hand rule is explained as follows, and the left-hand rule can be understood in a similar way.
To design the right-hand rule of always walking right, we define the priority of action selection as $RF>F>LF>LR>R>RR$.
That is, in any state $s_t$, the mobile robot will first try to choose the $RF$ action if it can pass through the right front direction. 
If not, the mobile robot will try to select the action in the order of $F, LF, LR, R, RR$ until it can find a direction to take a valid step.
The action selection strategy is executed following the right-hand rule until the mobile robot navigates to the goal point. 
We record the sequence of the states and actions as the right-hand trajectory $\mathcal{T}_r$ and the left-hand trajectory $\mathcal{T}_l$ as
\begin{equation}
	\begin{aligned}
		&\mathcal{T}_r=\{s_{r_1}, a_{r_1}, s_{r_2}, a_{r_2}, ..., s_{r_p}\}, \\
		&\mathcal{T}_l=\{s_{l_1}, a_{l_1}, s_{l_2}, a_{l_2}, ..., s_{l_q}\},
	\end{aligned}
	\label{tra}
\end{equation}
where $p$ and $q$ are the numbers of the traversed states in the right and the left trajectories, respectively.
$s_{r_p}$ and $s_{l_q}$ are the same goal state.
To better illustrate the left- and the right-hand rules, Fig.~\ref{fig:hand} presents a simple example of navigating in a hexagonal grid map, and Algorithm~\ref{algo:hand} summarizes the rule for generating the closed-loop trajectory.

\begin{figure}[tb]
	\centering
	\subfigure[Left-hand rule]{\includegraphics[width=0.43\columnwidth]{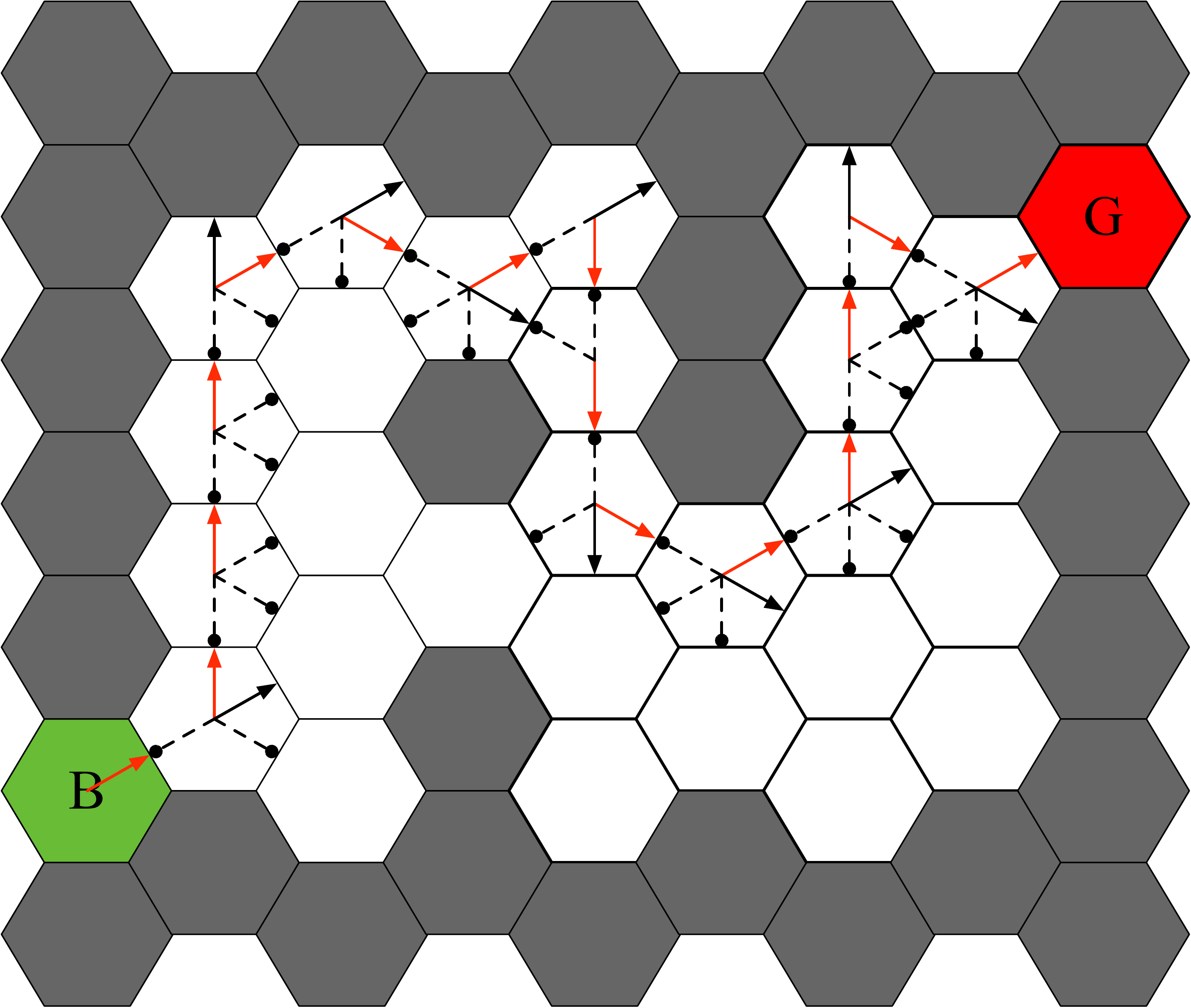}}
	\subfigure[Right-hand rule]{\includegraphics[width=0.43\columnwidth]{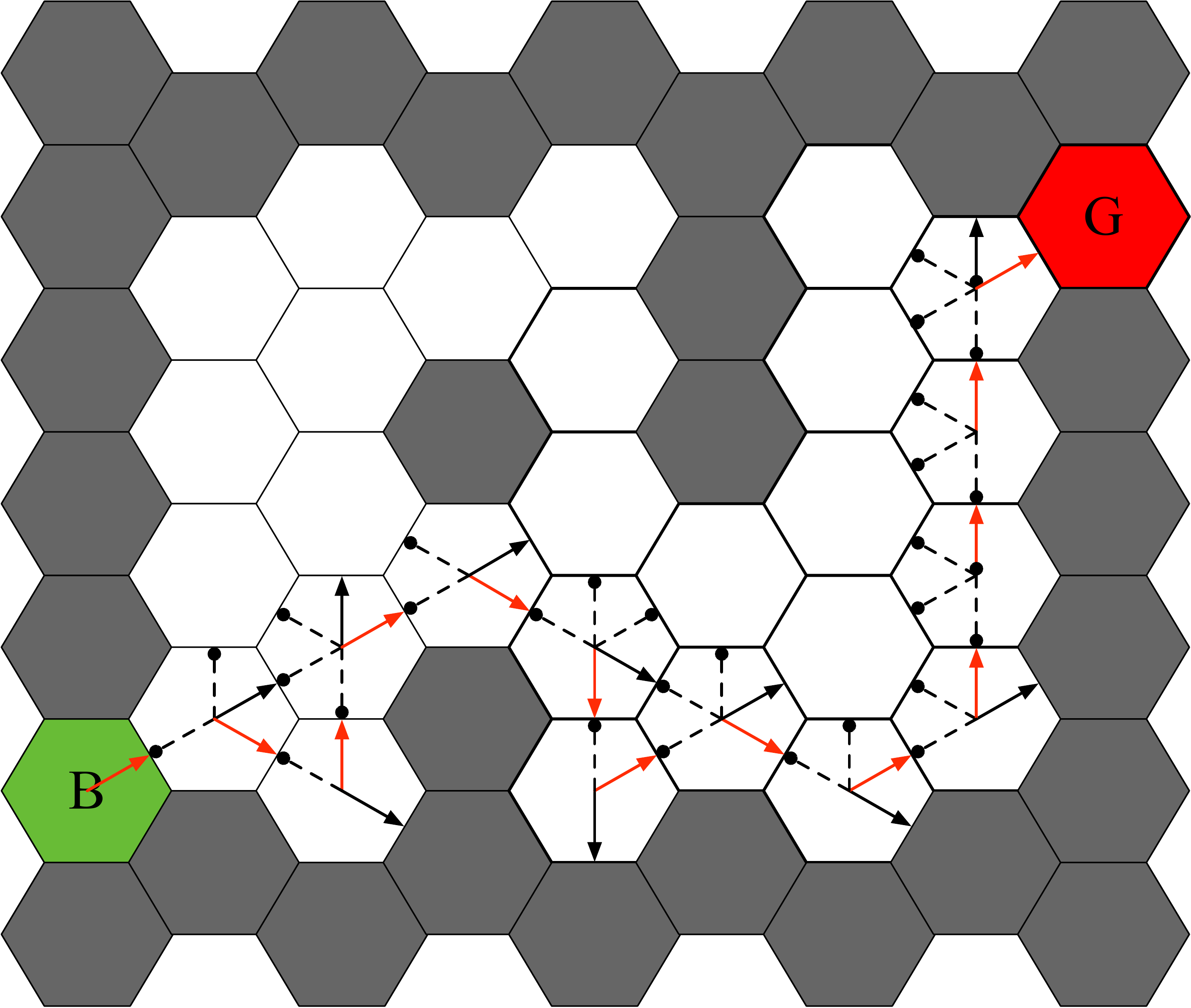}}
	\caption{A simple example of using the left- and right-hand rules to navigate in a hexagonal grid map. 
		$B$ is the starting point and $G$ is the goal point.
		The selected actions are indicated by the red arrows. 
		The other available actions are indicated by the dotted solid circular arrows.
		The main angle of view is indicated by the black arrows.}
	\label{fig:hand}
\end{figure}

\begin{algorithm}[tb]
	\caption{Right-hand (Left-hand) Rule}
	\label{algo:hand}
	\KwIn{Starting point $B$; Goal point $G$}
	\KwOut{Right-hand (left-hand) trajectory $\mathcal{T}_r$ ($\mathcal{T}_l$)}
	Initialize $\mathcal{T}_r=\emptyset$ ($\mathcal{T}_l=\emptyset$) \\
	Initialize $t=0, s_t$ \\
	\While{$s_t$ is not terminal}{
		Select $a_t$ according the right-hand (left-hand) rule \\
		Execute $a_t$, observe $s_{t+1}$ \\
		Add $(s_t, a_t)$ to $\mathcal{T}_r$ ($\mathcal{T}_l$ ) using Eq.~(\ref{tra})\\
		$t\leftarrow t+1$
	}
\end{algorithm}

\subsection{Rule for Reducing the Exploration Space}\label{Sec3.3}
After obtaining the left- and the right-hand trajectories, the available states that the mobile robot can access are on or inside the closed loop.
Obviously, the trajectory itself, $\mathcal{T}_l$ or $\mathcal{T}_r$, is a feasible while not necessarily optimal path that navigates from the starting point to the goal. 
If we can properly reduce the length of the left- and right-hand trajectories, we will obtain a smaller closed-loop trajectory that can effectively reduce the redundant exploration space. 
Hence, we employ the reduction rule to optimize the two trajectories $\mathcal{T}_l$ and $\mathcal{T}_{r}$, respectively.
The reduction process on the right-hand trajectory is explained as follows, and the operation on the left-hand trajectory can be understood in a similar way.

Our main idea is that, given two states on the right-hand trajectory, we aim to find out whether there exists a shorter path between the two states than the path on the trajectory.
If it does exist, we can replace the original path on the trajectory with the shorter one, thus obtaining a new trajectory with a reduced length.
To implement this idea, we need to first formally define the step distance between two given states and the trajectory distance between two states on the trajectory.
\begin{definition}[Step Distance]
	The step distance is defined as the number of the least steps of actions needed to transit from $s$ to $s'$ (analogous to the definition in~\citep{thomason2020vision}).
	Specifically, the step distance between the same states is 0 and the step distance from one state to its adjacent states is set as 1.
\end{definition}

\begin{definition}[Trajectory Distance]
	Given two states $s$ and $s'$ on the left- or the right-hand trajectory, their trajectory distance is defined as the number of the least steps of actions needed to transit from $s$ to $s'$, while the intermediate states should also be on the trajectory.
\end{definition}

Based on the definition of step distance and the fact that there are $6K$ $K$-step hexagonal grids around the center grid, we further define the $K$-step reachable states.
\begin{definition}[$K$-Step Reachable States]
	Given the state $s=(i,j)$, a state $s'$ is called the $K$-step reachable state of $s$ if the step distance between $s'$ and $s$ is $K$. 
	There exist $6K$ states whose step distance from $s$ is $K$, and these $6K$ states are called the $K$-step reachable states of $s$. 
\end{definition}

Based on the definition of $K$-step reachable states, we are able to easily figure out the potentially shorter path between two states on the trajectory.
Given a state $s$ on the right-hand trajectory, we first obtain its $K$-step reachable states.
For example, when $K$=2, those 12 states are $(i-4, j), (i-3, j+1), (i-2, j+2), (i, j+2), (i+2, j+2), (i+3, j+1), (i+4, j+1), (i+3, j-1), (i+2, j-2), (i, j-2), (i-2, j-2), (i-3, j-1)$.
If any reachable state is on the $\mathcal{T}_{temp}$, which is the sequence after the current state on the right-hand trajectory, we compute the trajectory distance from the given state to the reachable state.
If the trajectory distance is greater than $K$, it indicates that we have discovered a shorter path between them instead of the original path on the trajectory.
Hence, we can replace the original path with the new $K$-step path $\mathcal{T}_j$, resulting in an improved right-hand trajectory.\footnote{When the optimized step size is $K$, there are $2^K-1$ $K$-step reachable paths, and $\mathcal{T}_j$ is one of them.}
We apply this reduction rule for every state in the right-hand trajectory, and obtain an optimized trajectory $\mathcal{T}_r^K$. 
In similar way, we can obtain the optimized left-hand trajectory $\mathcal{T}_l^K$.
Together, a smaller closed-loop trajectory is formed.
Algorithm~\ref{algo:reduction} summarizes the rule for reducing the exploration space.

\begin{algorithm}[tb]
	\caption{Rule for Reducing Exploration Space}
	\label{algo:reduction}
	\KwIn{Optimization step $K$; \newline left- and right-hand trajectories $\mathcal{T}_{l}, \mathcal{T}_{r}$} 
	\KwOut{Reduced trajectories $\mathcal{T}_{l}^K, \mathcal{T}_{r}^K$}
	\For {$\mathcal{T}_{opt}$ = $\mathcal{T}_{l}, \mathcal{T}_{r}$} {
	  Get $l$ from the length of $\mathcal{T}_{opt}$\\ 
	   \While{$l$ is changing}{
		Update the length of $l$   \\
		\For{$i=1,...,l-K$}{
			Get $\mathcal{T}_{temp}$ after the $i$-th state in $\mathcal{T}_{opt}$\\
			Obtain the $K$-step reachable states $\mathcal{T}_{reach}^K$\\
			\If{$\mathcal{T}_{reach}^K$ matches state $s_{j}$ in  $\mathcal{T}_{temp}$ }{
				Replace the sequence from $s_{i}$ to $s_{j}$ in $\mathcal{T}_{opt}$ with  the new $K$-step path $\mathcal{T}_{j}$ \\
				Update  $\mathcal{T}_{opt}$ 
			}
		}
	}
	$\mathcal{T}_{l}^K,\mathcal{T}_{r}^K$ ${\leftarrow} \mathcal{T}_{opt}$
	}
\end{algorithm}

Fig.~\ref{fig:reduction} presents a simple example of optimizing the closed loop trajectory.
In Fig.~\ref{fig:reduction}(a), the path trajectory has been marked as a light gray area, in which actions are selected using the left-hand rule. 
In Fig.~\ref{fig:reduction}(b), from the start state to the end state of the trajectory obtained in Fig.~\ref{fig:reduction}(a), we sequentially generate a 1-step reachable state $\mathcal{T}_{reach}^1$ for each state and match it with the subsequent trajectory. 
If the match is successful, we replace it with a new 1-step path, and otherwise we do the same for the next state.
For example, the original path is ${C}\rightarrow{C2}\rightarrow{C3}$ and ${D}\rightarrow{D1}\rightarrow{D2}$, which can be optimized as ${C}\rightarrow{C3}$ and ${D}\rightarrow{D2}$.  
After optimization, we obtain the optimized closed-loop space which is inside the light gray trajectories, as shown in Fig.~\ref{fig:reduction}(c).
In the closed-loop space formed after the optimization of $K=1$, we will optimize with $K=2$.
The process is similar to that in Figs.~\ref{fig:reduction}(a)-(c), as a detailed process shown in Figs.~\ref{fig:reduction}(d)-(f).

\begin{figure*}[tb]
	\subfigure[]{\includegraphics[width=0.22\textwidth]{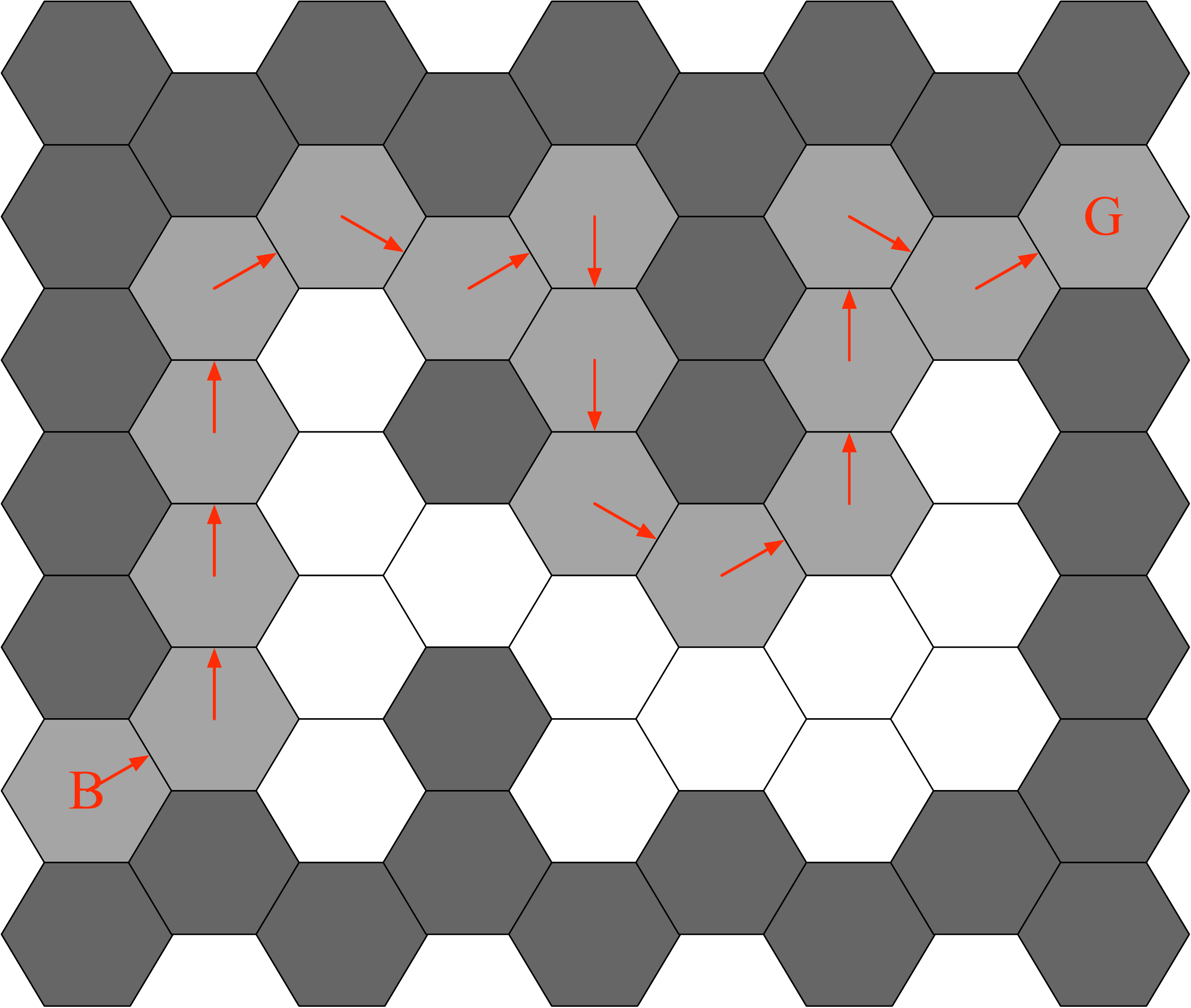}}
	\subfigure[]{\includegraphics[width=0.31\textwidth]{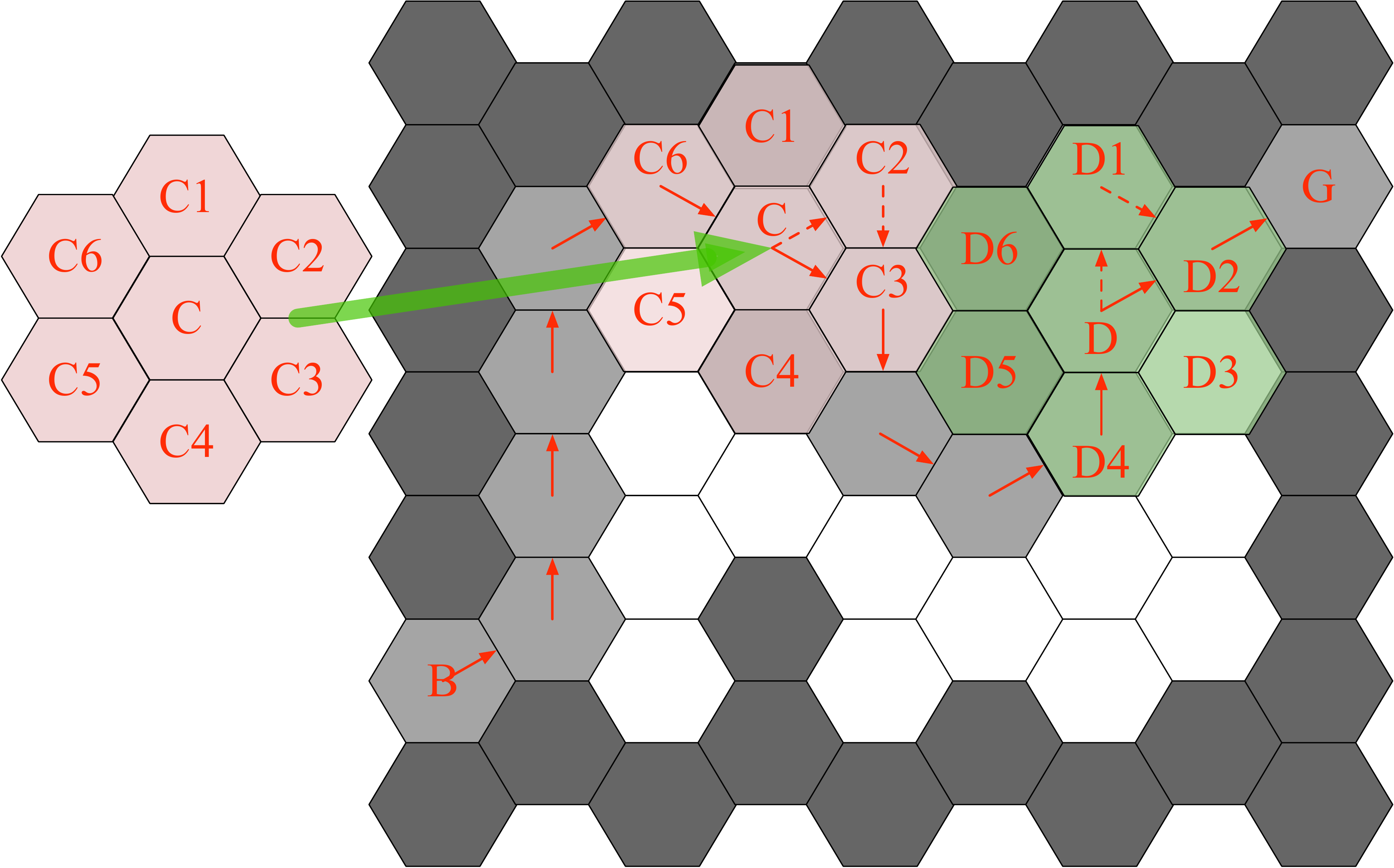}}
	\subfigure[]{\includegraphics[width=0.23\textwidth]{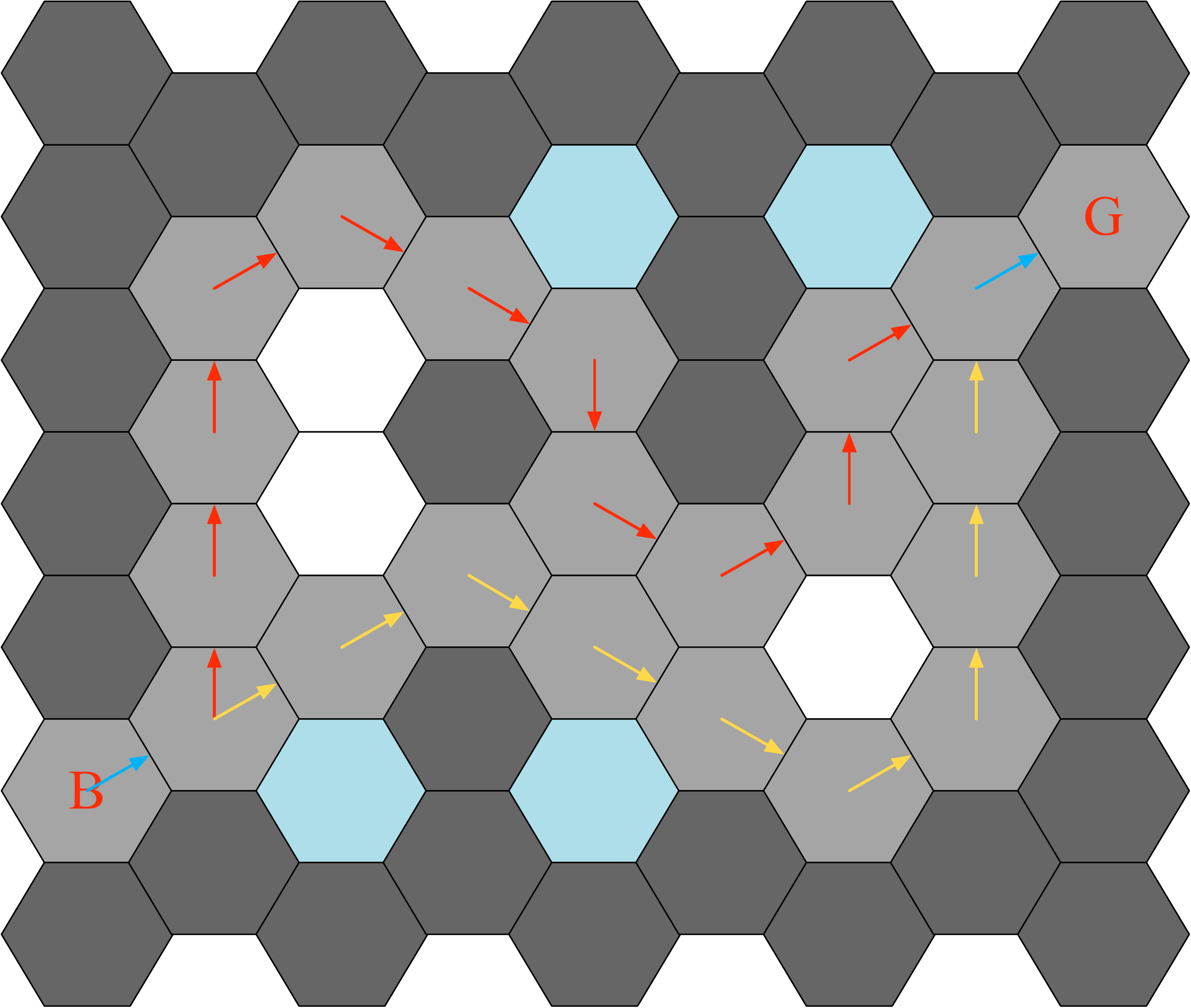}}
	\subfigure[]{\includegraphics[width=0.23\textwidth]{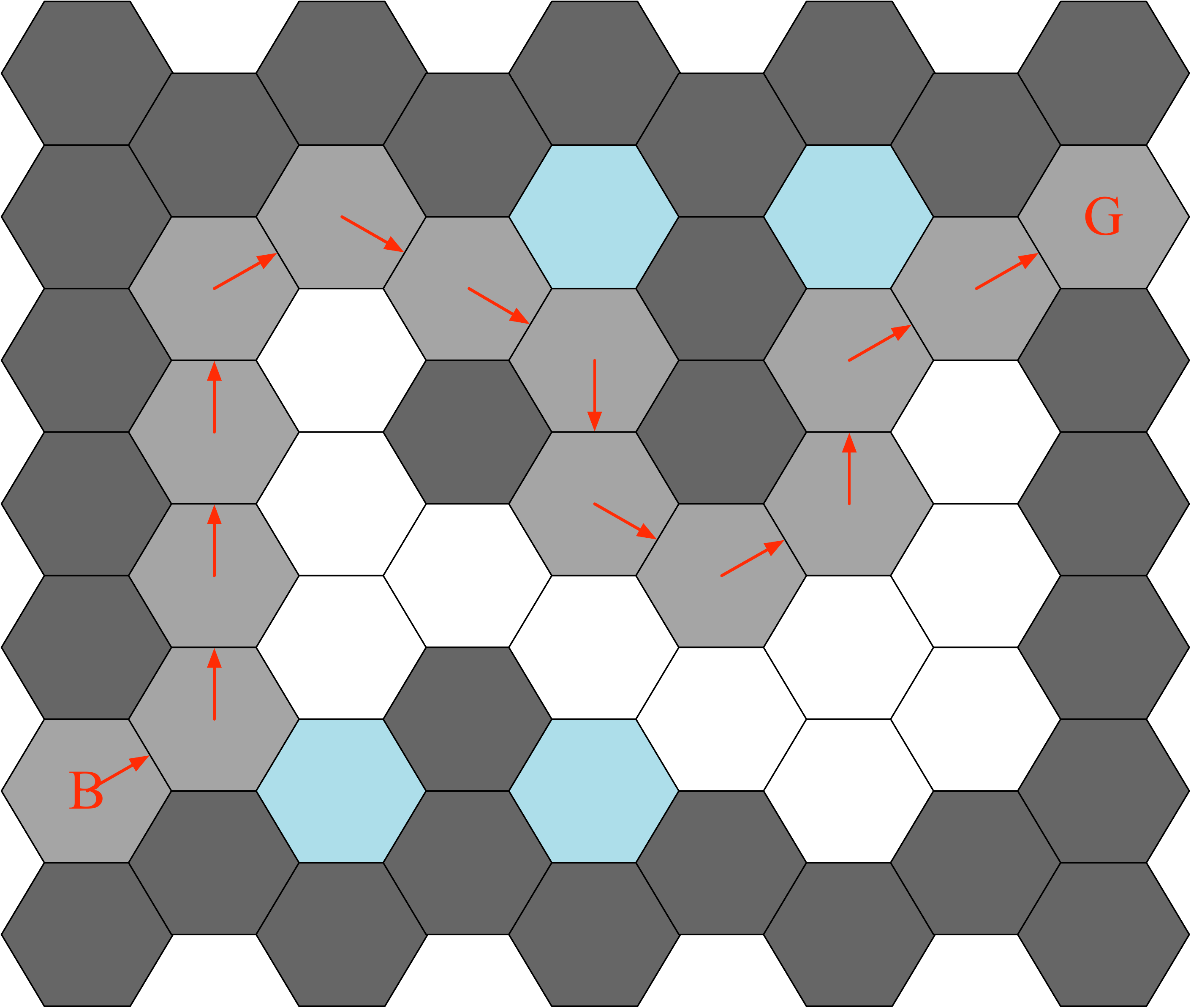}}
	\hspace{11mm}
	\subfigure[]{\includegraphics[width=0.35\textwidth]{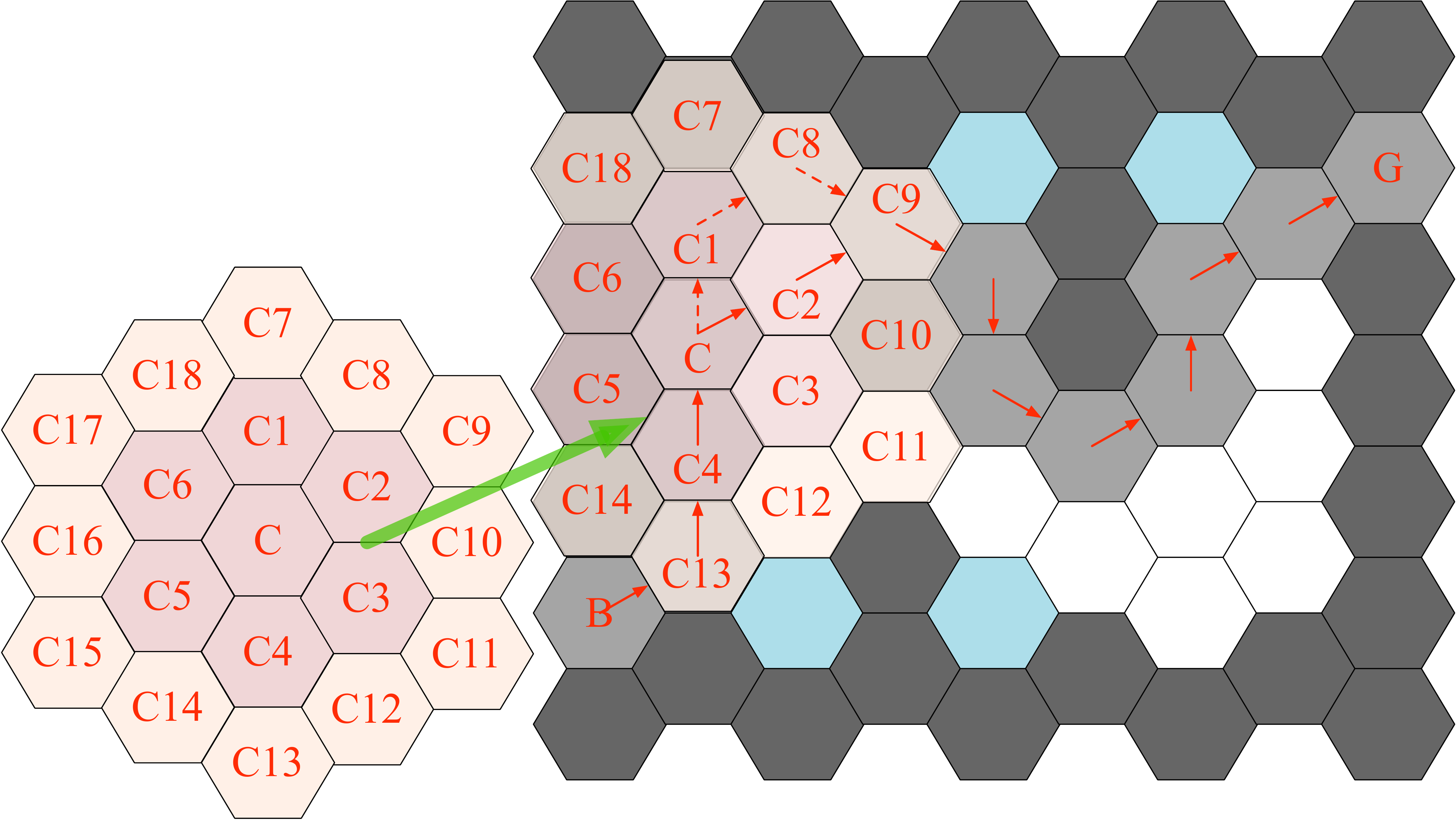}}
	\hspace{20mm}
	\subfigure[]{\includegraphics[width=0.23\textwidth]{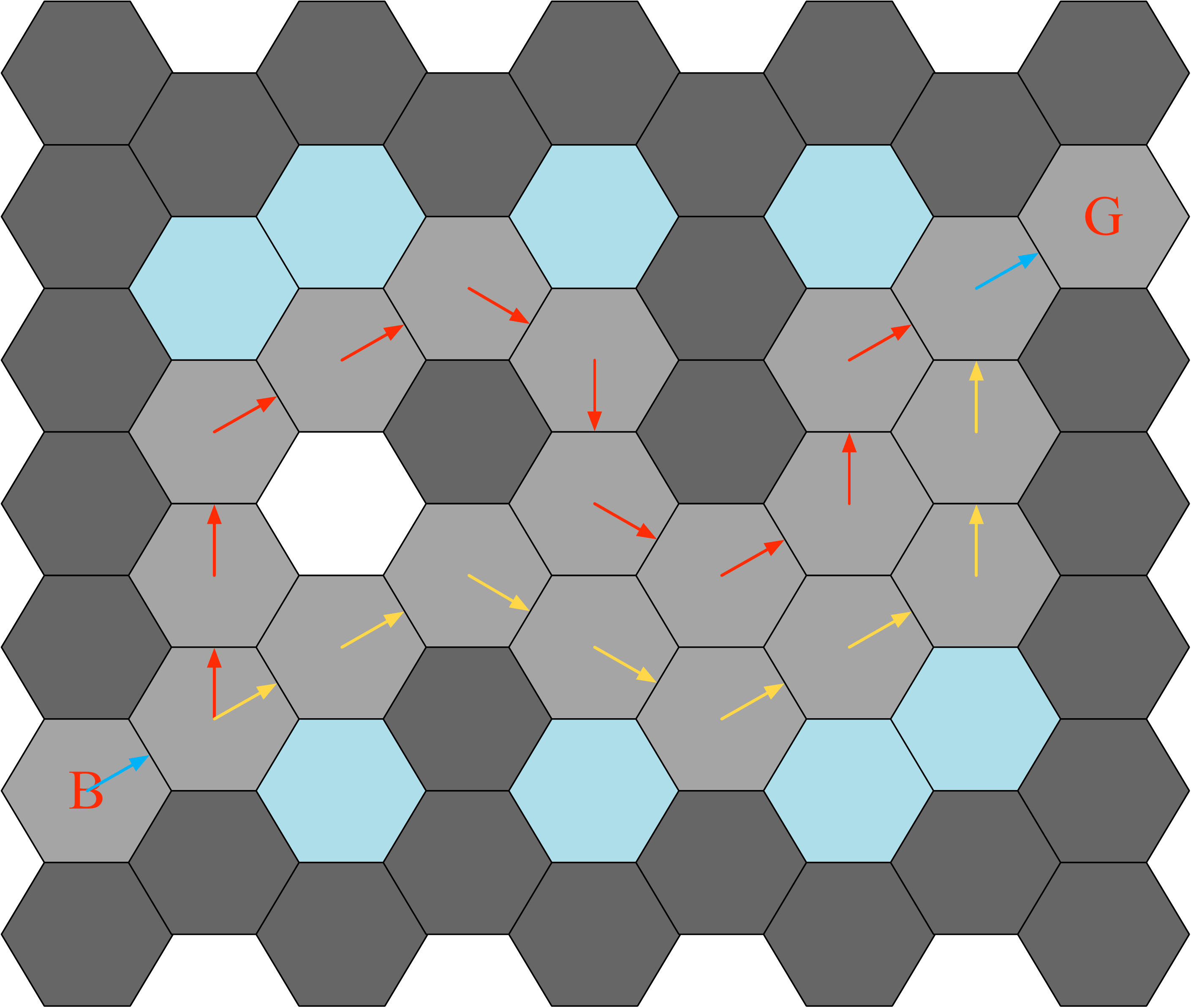}}
	\caption{A simple example of optimizing the closed-loop trajectory using different optimization steps of $K=1$ (a)-(c) and $K=2$ (d)-(f). The selected actions are indicated as red or yellow solid arrows. The actions before optimizing the trajectory are indicated as dotted arrows. The closed-up view on the left side of (b) and (e) shows the $K$-Step reachable states. The reduced exploration space is indicated as light blue hexagonal grids.}
	\label{fig:reduction}
\end{figure*}

Intuitively, the optimized trajectories can form an improved closed loop for more efficient exploration. 
Further, Theorem~1 presents the theoretical analysis that the optimized trajectories correctly reduce the redundant exploration space.

\begin{theorem}
	Let $s$, $s'$ be two states on the right-hand trajectory, with their step distance being $K$ and their trajectory distance being $J$.
	Let $v(s')$ be the value function of state $s'$.
	Let $\pi_1$ and $\pi_2$ denote the policies of navigating from state $s$ to state $s'$ following the original path on the right-hand trajectory $\mathcal{T}_r$ and the optimized path on the optimized right-hand trajectory $\mathcal{T}_r^K$, respectively.
	$v_{\pi_1}(s)$ and $v_{\pi_2}(s)$ are the value functions of state $s$ when executing policy $\pi_1$ and $\pi_2$, respectively.
	Then, for any $J\ge K$, we have $v_{\pi_1}(s)\le v_{\pi_2}(s)$. 
\end{theorem}

\begin{proof}
	When executing policy $\pi_1$, assume that $s_{\pi_1}^1, ..., s_{\pi_1}^{J-1}$ are the sequential states between states $s$ and $s'$. According to the Bellman equation~\citep{bellman2013dynamic,sutton2018reinforcement}, the value function of state $s$ is
	\begin{equation}\nonumber
		\begin{aligned}
			v_{\pi_1}(s) & = \mathbb{E}_{\pi_1}\left[\sum_{i=0}^{\infty}\gamma^ir_i|s_0=s\right] \\
			& = \sum_{a}\pi_1(a|s)\sum_{s''}p(s''|s,a)(r + \gamma v_{\pi_1}(s'')) \\
			& = r + \gamma v_{\pi_1}(s_{\pi_1}^1) = r + \gamma(r + \gamma v_{\pi_1}(s_{\pi_1}^2)) = ... \\
			& = (r + \gamma r + ... + \gamma^{J-1} r) + \gamma^{J-1}v_{\pi_1}(s') \\
			& = \frac{1-\gamma^J}{1-\gamma}r + \gamma^{J-1}v_{\pi_1}(s').
		\end{aligned}
	\end{equation}
	In a similar way, the value function of state $s$ when executing policy $\pi_2$ is
	\begin{equation}\nonumber
		v_{\pi_2}(s) = \frac{1-\gamma^K}{1-\gamma}r + \gamma^{K-1}v_{\pi_2}(s').
	\end{equation}
	We can set state $s'$ as the terminal state. 
	Then, we have $v_{\pi_1}(s')=v_{\pi_2}(s')=0$, and
	\begin{equation}\nonumber
		v_{\pi_1}(s) - v_{\pi_2}(s) = \frac{\gamma^K - \gamma^J}{1-\gamma}r.
	\end{equation}
	Since $J\ge K$, $0\le\gamma\le 1$, and $r\le 0$ in the navigation domains, we have $v_{\pi_1}(s) \le v_{\pi_2}(s)$.
\end{proof}

Theorem~1 proves the effectiveness of trajectory optimization. 
Suppose that $\mathcal{T}^*$ is the optimal path in the original closed loop formed by trajectories $\mathcal{T}_l$ and $\mathcal{T}_r$.
Next, we present and prove Theorem~2 which demonstrates that the optimal path $\mathcal{T}^*$ is still within the closed loop formed by $\mathcal{T}_l^K$ and $\mathcal{T}_r^K$.

\begin{theorem}
    Let $\mathcal{C}$ denote the closed-loop region formed by $\mathcal{T}_l^K$ and $\mathcal{T}_r^K$.
    For any optimal path $\mathcal{T}^{*}$, we have $\mathcal{T}^{*}\subseteq \mathcal{C}$. 
\end{theorem}

\begin{proof}
	We assume that there is an optimal path $\mathcal{T}^{*}$, and $\mathcal{T}^{*}\nsubseteq \mathcal{C}$.
    Since both the starting and goal points are on $\mathcal{C}$, $\mathcal{T}^{*}$ and $\mathcal{T}_l^K$ (or $\mathcal{T}^{*}$ and $\mathcal{T}_r^K$) intersect at least twice.
    We assume that the two intersecting positions are $D$ and $G$, as shown in Fig.~\ref{fig:throrem}.
    The trajectory distances from $B$ to $G$ on $\mathcal{T}_l^K$ and $\mathcal{T}^{*}$ are $J$ and $I$, respectively. 
    The trajectory distances from $D$ to $G$ on $\mathcal{T}_l^K$ and $\mathcal{T}^{*}$ are $j$ and $i$, respectively. 
    Let $\pi^*, \pi_1$, and $\pi_2$ denote policies of navigating from state $B$ to state $G$ following the path on $\mathcal{T}^*$, the path on $\mathcal{T}_l^K$, and the path composed of the trajectory ${B}\rightarrow{D}$ on $\mathcal{T}^*$ and ${D}\rightarrow{G}$ on $\mathcal{T}_l^K$, respectively.
    Since $\mathcal{T}^*$ is the optimal path, we have $I < J$, $v_{\pi^*}(B) > v_{\pi_1}(B)$, and $v_{\pi^*}(B) > v_{\pi_2}(B)$.
    According to the reduction rule, we have $j<i$.
    Further, we have $v_{\pi_2}(D) > v_{\pi^*}(D)$ according to Theorem~1.
    Since $\pi^*$ and $\pi_2$ share the same subpath from B to D, we have $v_{\pi_2}(B) > v_{\pi^*}(B)$.
    It is contradictory to the assumption that $\mathcal{T}^*$ is the optimal path.
    Hence, the assumption does not hold. 
    Then, for any $\mathcal{T}^*$, we have $\mathcal{T}^*\subseteq \mathcal{C}$. 
\end{proof}
\begin{figure}[tb]
	\centering
	\subfigure{\includegraphics[width=0.6\columnwidth]{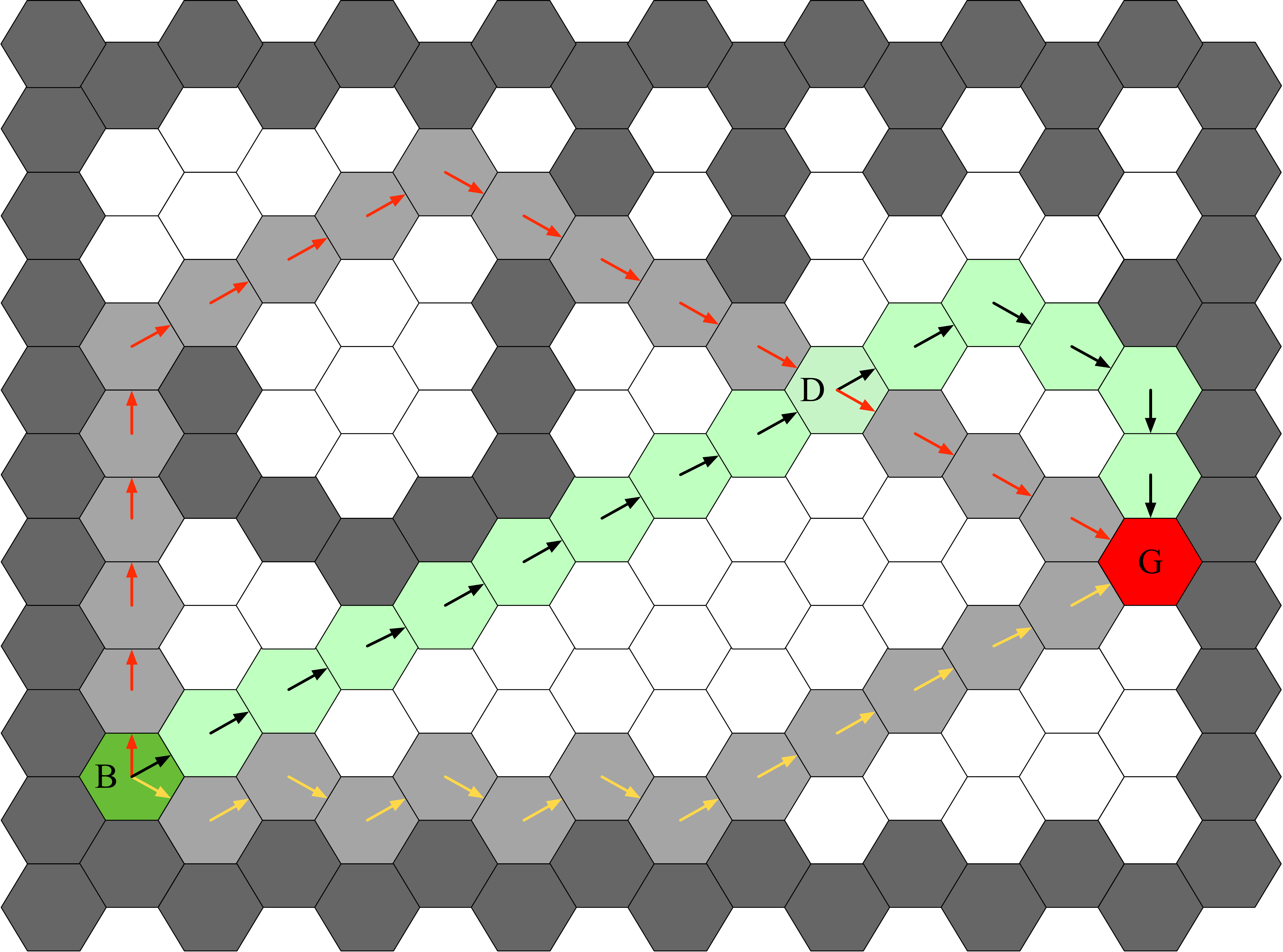}}
	\caption{An illustrative example for proving that the optimal path is still within the closed-loop $\mathcal{C}$.  
		The trajectories $\mathcal{T}_{l}^K, \mathcal{T}_{r}^K$ are indicated as red and yellow solid arrows.
		The optimal path assumed $\mathcal{T}^{*}$ is indicated as black solid arrows.
		The intersecting positions of $\mathcal{T}_{l}^K, \mathcal{T}^{*}$ are indicated as $D$ and $G$.}
	\label{fig:throrem}
\end{figure}

When the optimization step is $K$, there are $6K$ hexagon grids around the center point to be optimized.
Let $n$ denote the length of the trajectory to be optimized.
Then, the complexity of the optimization algorithm is $6nK$, i.e., $\mathrm{O}(nK)$. 
It can be observed that selecting $K$ with an enormous value will linearly increase the computational cost of trajectory reduction.
On the other hand, a larger $K$ generally leads to a smaller closed-loop trajectory, which can reduce the redundant exploration space to a more considerable extent. 
In practice, a moderate value of $K$ (e.g., $2\sim 4$) is sufficient to obtain efficient performance.

\subsection{Rule for Guiding the Early Exploration Strategy}\label{Sec3.4} 
In the reduced exploration space, we use the closed loop formed by the optimized trajectories as the new navigation environment.
While the agent can learn more efficiently in the reduced space, the agent still needs to explore many steps to find the goal point at the early stage.
We employ the Pledge rule~\cite{abelson1986turtle} to accelerate the early learning performance when the number of steps exceeds a threshold value without finding the goal. 
The counter-clockwise method of the Pledge rule is explained as follows.
Similarly, the clockwise way operates.

In the Pledge rule (counter-clockwise), the agent first needs to choose an initial action direction, and moves towards this action direction with priority.
Next, to meet the priority of the initial action direction selected, we employ the sum of turns \(\theta\) to record the changes of the action direction and update the main angle view of the mobile robot by the previous action $a_{past}$.
When an obstacle is met, the sum of turns \(\theta\) is added by 1 per 60 degrees if the clockwise turn is positive, and is subtracted by 1 per 60 degrees otherwise.
Finally, for the purpose of avoiding traps, if the overall turning angle \(\theta\) is $0$, the agent will take action in the priority order of $F>LF>LR>R>RR>RF$ (defined as \(\Theta_{0}\) rule).
Otherwise, the agent will choose an action in the priority order of $RF>F>LF>LR>R>RR$, which operates in the same way as the right-hand rule in Section \ref{Sec3.2}. 
By recording the sum of the turns $\theta$ and keeping the initial direction, the Pledge algorithm can find the goal point, regardless of the initial position of the agent.
The clockwise method that counts the overall turning angle \(\theta\) is opposite to the counter-clockwise method, and the Pledge algorithm is summarized in Algorithm~\ref{pledge}.

\begin{algorithm}[tb]
\caption{Pledge Algorithm (Counter-clockwise)}
\label{pledge}
\KwIn  {Sum of turns \(\theta\); current state $s$; turns \(\theta^{'}\) at state $s$; previous action $a_{past}$; goal point $G$}
\KwOut {Action $a$; updated \(\theta\)} 
Initialize \(\theta^{'}=0\), $a_{past}$ \\
\eIf {\(\theta=0\) at state $s$} {
       Select $a$ according to \(\Theta_{0}\) rule and $a_{past}$ \\
       Calculate the turns \(\theta^{'}\)\\
       \(\theta\) += \(\theta^{'}\)\\
       \(a_{past} \leftarrow a\)}
       {
       Select $a$ according the right-hand rule and $a_{past}$ \\
       Calculate the turns \(\theta^{'}\)\\
       \(\theta\) += \(\theta^{'}\)\\
       \(a_{past} \leftarrow a\)}
\end{algorithm}

To utilize the Pledge algorithm to improve the exploration efficiency, we use the counter-clockwise method when the number of learning episodes $\eta$ is odd and employ the clockwise method otherwise.
We design a decay function for the threshold value of learning steps as 
\begin{equation}
	{E}=\frac{{M}_{max}}{\omega * \eta+b},
\label{Exe}
\end{equation}
where ${M}_{max}$ is the maximum learning steps per episode and $\eta$ is the number of the current episode. 
$b$ tunes the expected number of threshold steps in previous episodes, and $\omega$ controls the decay rate.
Generally, a smooth decay function (e.g., $\omega = 0.15 $ and $b=10$) can obtain effective improvement with the Pledge rule for guiding the agent to explore.

\subsection{Integrated RuRL Algorithm}\label{Sec3.5}
With the above implementations, the integrated RuRL for navigation algorithm is summarized in Algorithm~\ref{rurl}.
First, in the original environment, the RL agent obtains trajectories $\mathcal{T}_{l}$, $\mathcal{T}_{r}$ through left- and right-hand rules in Line 1.
Second, we employ the rule for optimizing the initial trajectories and reducing exploration space in Lines 2-3, and generate a smaller navigation environment.
Third, in the new environment with the reduced space, we employ RL to learn the optimal policies in Lines 4-21.
Finally, we employ the Pledge rule to accelerate the early learning performance for a small number of episodes when the number of steps exceeds a threshold value without finding the goal in Lines 8-13. 

\begin{algorithm}[tb]
\caption{RuRL Algorithm}
\label{rurl}
\KwIn {Learning rate $\alpha$; small \(\varepsilon\); Pledge rule usage episodes $N$; threshold steps of each episode $E$; discounting factor \(\gamma\); optimization step $K$}
\KwOut {Optimal policy \(\pi^{*}\)}
$\mathcal{T}_{l}, \mathcal{T}_{r}$ ${\leftarrow}$ Get trajectories using Algorithm 1\\
$\mathcal{T}_{l}^K, \mathcal{T}_{r}^K$ ${\leftarrow}$ Optimal the trajectories using Algorithm 2\\
Env ${\leftarrow}$ Connect two optimized trajectories $\mathcal{T}_{l}^K, \mathcal{T}_{r}^K$\\
Initialize $Q(s, a)$ arbitrarily, $\forall s \in S, a \in {A}(s)$ \\
\For (each episode){$\eta$ up to $T_{max}$} {
          Initialize $s$\\
\For (each step of episode){\(s\) is terminal} { 
\eIf {$\eta$ \(<=N\) and steps \(>=E\) using Eq.~(\ref{Exe})}  {
\eIf {$\eta$ is odd} {
  Choose \(a\) using policy derived from Pledge Rule (counter-clockwise)} {
 Choose \(a\) using policy derived from Pledge Rule (clockwise)}
} {
 Choose \(a\) using policy derived from $Q$}
Take action \(a,\) observe \(r, s^{\prime}\)\\
$Q(s,a) \!\leftarrow\! Q(s,a)\!+\!\alpha[r\!+\!\gamma\max_{a'}Q(s',a')\!-\!Q(s,a)]$ \\
\(s \leftarrow s^{\prime}\)\\
}
}
\end{algorithm}

\begin{remark} 
	When the starting point is not adjacent to the wall, the Pledge rule can solve this problem \cite{abelson1986turtle}.
	When the goal point is not adjacent to the wall, we can solve it by setting the sub-target point near the goal point.
	To highlight the usage of rules, we only consider the situation where both the starting and goal points are adjacent to the wall in this paper. 
\end{remark}

\section{Experiments}\label{Sec4}
We conduct two sets of experiments to evaluate the feasibility and effectiveness of RuRL.
One is the single-room navigation tasks consisting of the obstacle-free map and the map with obstacles.
The other is the multi-room navigation with a large and complex map, where conventional methods tend to explore inefficiently or converge to sub-optimal policies.

\subsection{Experimental Settings}\label{Sec4.1}
Since we aim to utilize rules to improve the exploration efficiency of RL in navigation domains, we focus on comparing RuRL to three baselines: RL without rules, RL with count-based exploration, and RL with UCB-based exploration~\citep{ostrovski2017count, saito2014discounted}.
We employ the learning curve and total learning steps as the performance metrics.
For single-room experiments, the Q-learning~\citep{sutton2018reinforcement} algorithm with the $\varepsilon$-greedy strategy is investigated to evaluate the effectiveness of RuRL.
Further, the Q-learning~\citep{sutton2018reinforcement} and SARSA~\citep{sutton2018reinforcement} algorithms with the $\varepsilon$-greedy and the Softmax exploration strategies are investigated in the complex task.
Besides, we compare our method with classic robot navigation methods, including A* and ACO heuristic algorithms, and utilize the length of the final planned path and the direction switching times of the path as the performance metrics.
More details can be found in~\citep{duchovn2014path} and \citep{chaari2012smartpath}.

In all environments, we use the double-width coordinate system, where the state is the 2D coordinate $(i, j)$, and available actions are: north towards $(i-2, j)$, northeast towards $(i-1, j+1)$, southeast towards $(i+1, j+1)$, south towards $(i+2, j)$, southwest towards $(i+1, j-1)$, and northwest towards $(i-1, j-1)$.
When robots collide with obstacles, they bounce to the previous position. 
For all experiments, the reward is $100$ if reaching the target, $-100$ if heading towards obstacles, and $-1$ otherwise.
The hyperparameters are the same for all tested algorithms in each group of experiments: learning rate $\alpha = 0.01$ and discount factor $\gamma = 0.99$.
Additionally, in the count-based exploration strategy, we use $N(s,a)$ to explicitly refer to the number of visits of a state-action pair in the learning process, and use an exploration bonus of the form $R^{+}(s, a)=\sqrt{\frac{\beta}{\log [\operatorname{N}(s, a)+1]}}$, where $\beta$ is set as a constant $0.4$. 
In RL with UCB-based method, we use the same exploration ratio settings as the $\varepsilon$-greedy strategy, and set the same parameters to all the tested algorithms: the damping factor $d=0.9$ and the tendency of exploration constant $C'=0.01$.
More details about parameter settings of RL with the UCB-based method can be found in~\citep{saito2014discounted}.
Additionally, the euclidean distance is used for the heuristic function in the A* algorithm.
In the ACO method, the number of ants is set as 100, and the other parameter settings can be found in~\citep{chaari2012smartpath}.
The navigation maps, constructed from real environments by a SLAM mobile robot with high-precision lidar, ultrasonic, and IMU sensors running on the ROS Kinetic platform, are hexagonally rasterized on the MATLAB simulation platform.
The environment perception control systems operate on an industrial PC (CPU: ARMv8 1.2GHz, GPU: 400MHz VideoCore IV).
All the experiments are carried out on an Intel Core i7-7700 3.60 GHz PC with 16G RAM under Windows 10.
The experimental results given are averaged over 50 runs.

\subsection{Tasks in Single-room Environments}\label{Sec4.2}
The real single-room environment is shown as in Fig.~\ref{fig:previous}(a).
Using a SLAM mobile robot, we constructed maps of the obstacle-free environment and the environment with obstacles as shown in Figs.~\ref{fig:previous}(b)-(c), respectively.
The maps are rasterized with hexagonal grids, where the gray indicates the unknown area, the black indicates obstacles, and the white indicates the feasible area. 
For both environments, the starting point is set as $(34, 17)$ with the green mark, and the goal point is set as $(3, 2)$ with the red mark.
The maximum number of learning episodes is set as $7000$, and the maximum number of learning steps per episode is set as ${M}_{max}=10000$.
The optimization step $K$ is set as 3 for reducing the exploration space.
The exploration ratio is $e^{-0.001 * \eta}(\eta<3500)$ and $0$ otherwise.
Parameters of the Pledge rule are set as $N=100$ and $E=\frac{10000}{0.2 * \eta+8}$.
Fig.~\ref{fig:result1} shows the learning curve and Table~\ref{Table2} presents corresponding numerical results.

\begin{figure}[tb]
\centering
\subfigure[Real Environment]{\includegraphics[width=0.5\columnwidth]{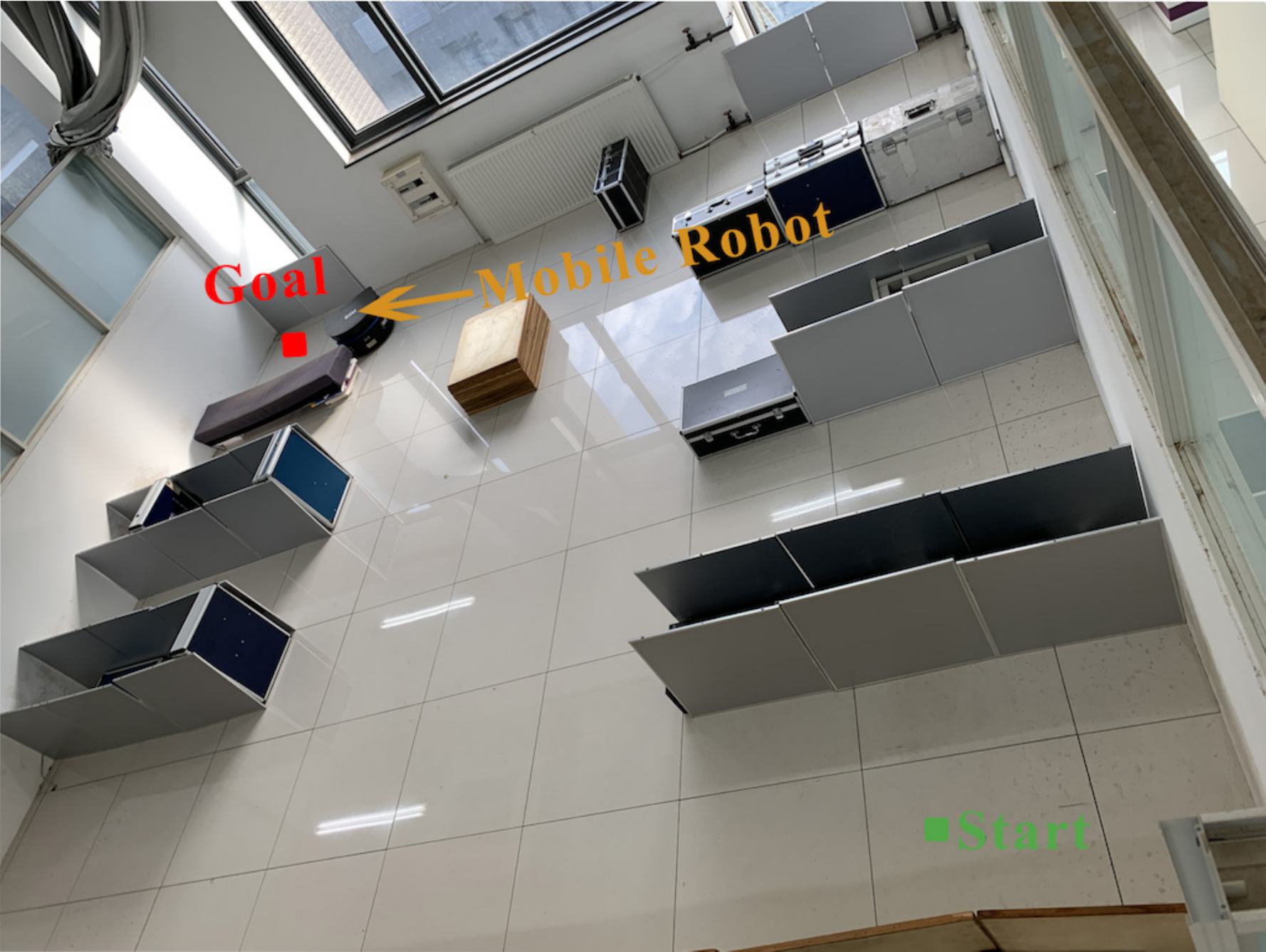}} \\
\subfigure[Env 1: obstacle-free]{\includegraphics[width=0.35\columnwidth]{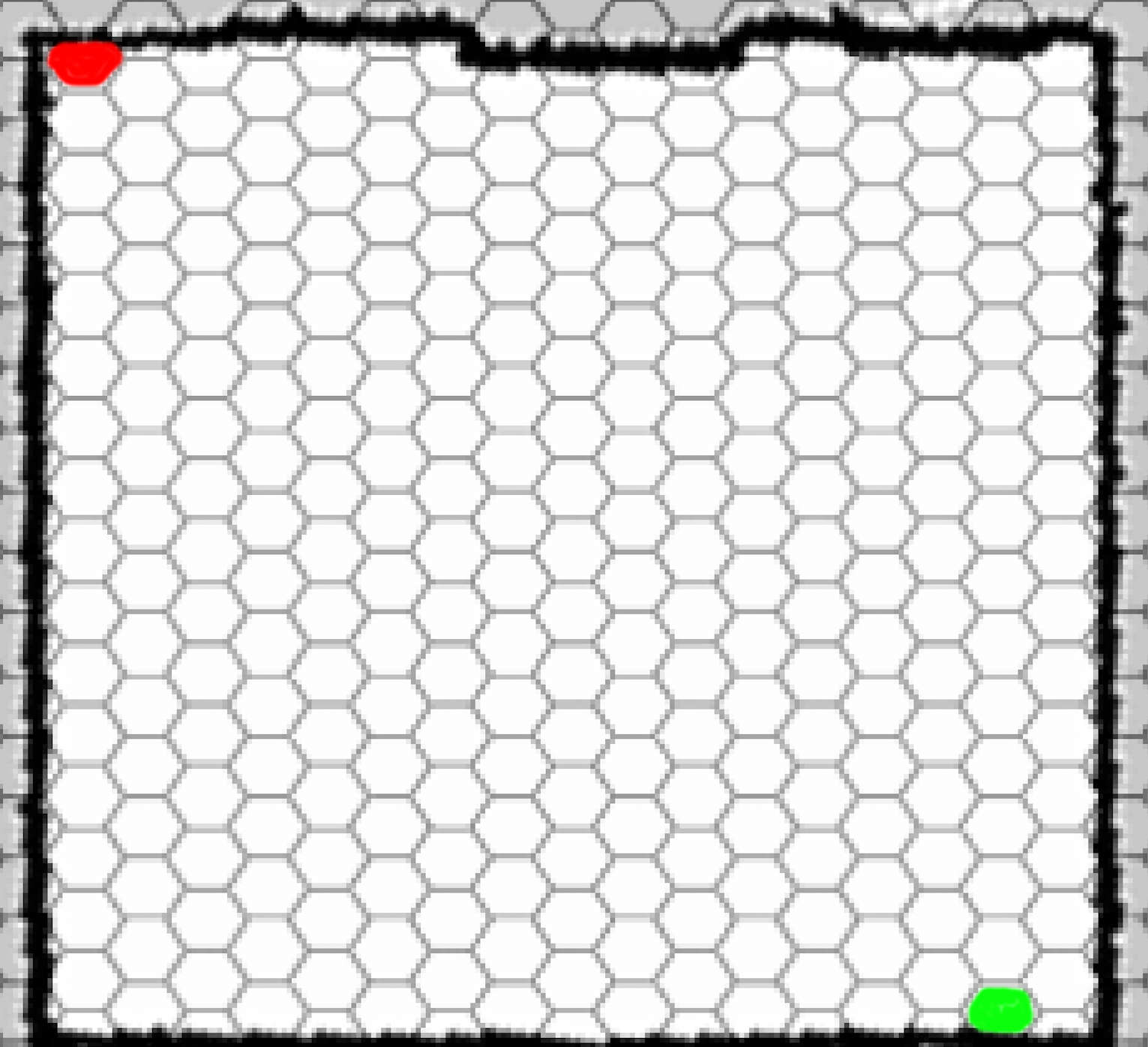}} \hspace{1em}
\subfigure[Env 2: with obstacles]{\includegraphics[width=0.33\columnwidth]{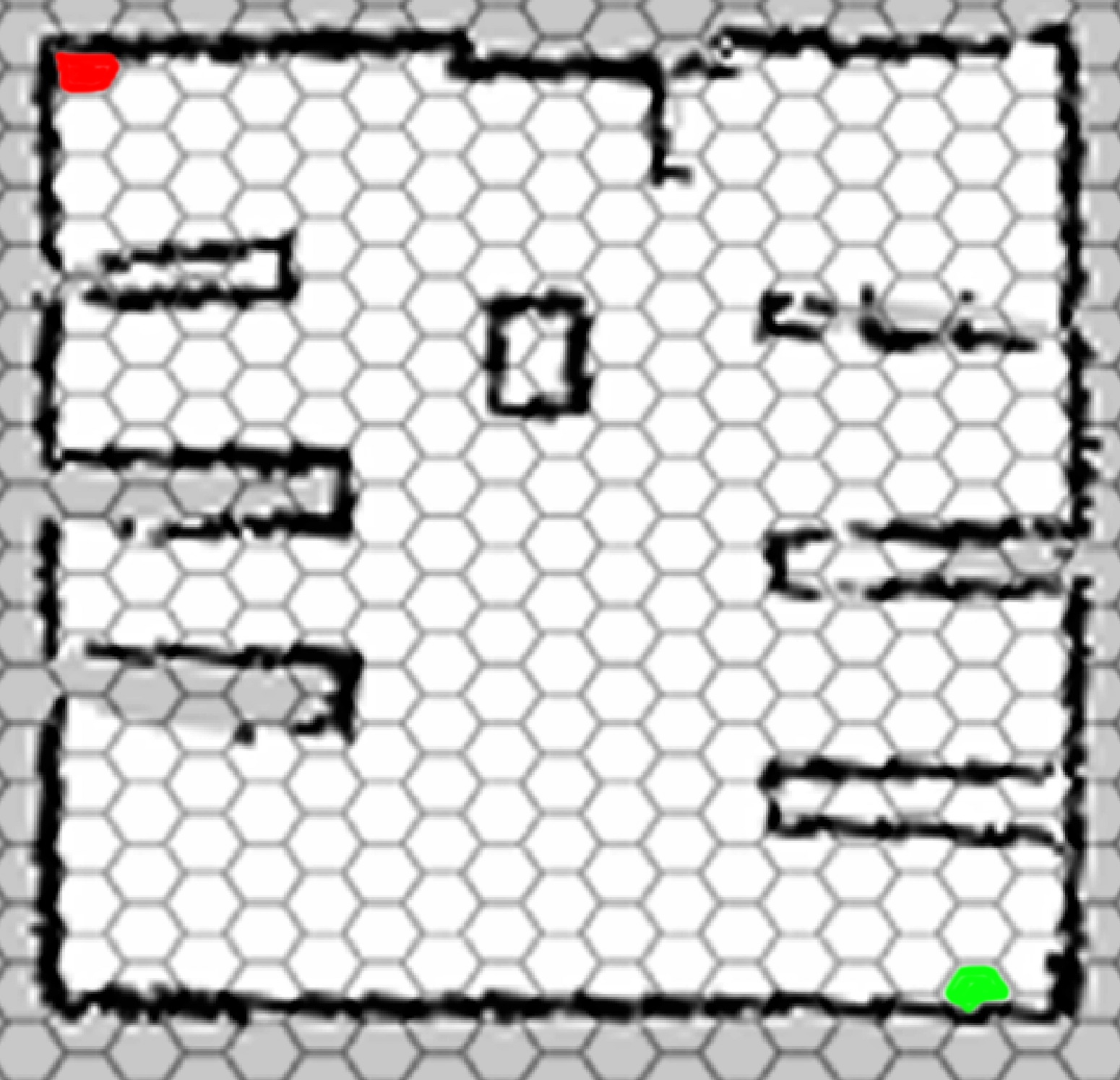}}
\caption{The physical size of the single-room navigation experimental environments is $465cm * 458cm$. Using Eq.~(\ref{equ:col}), the environment is rasterized into a hexagonal map with 35 rows and 19 columns after setting the grid length as $a=15.8cm$.}
\label{fig:previous}
\end{figure}

\begin{figure}[tb]
\centering
 \subfigure[Env 1]{\includegraphics[height=3cm]{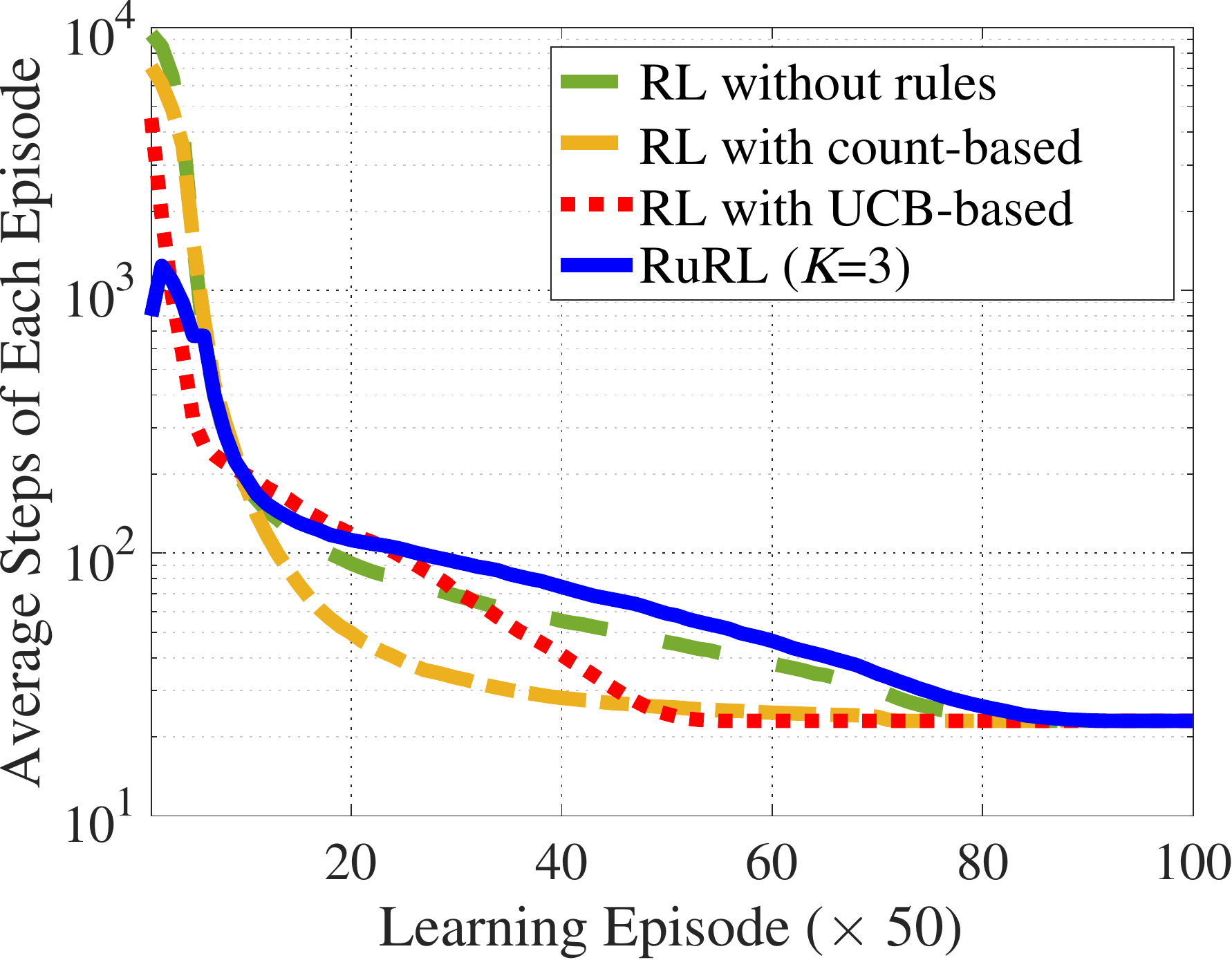}} \hspace{1em}
 \subfigure[Env 2]{\includegraphics[height=3cm]{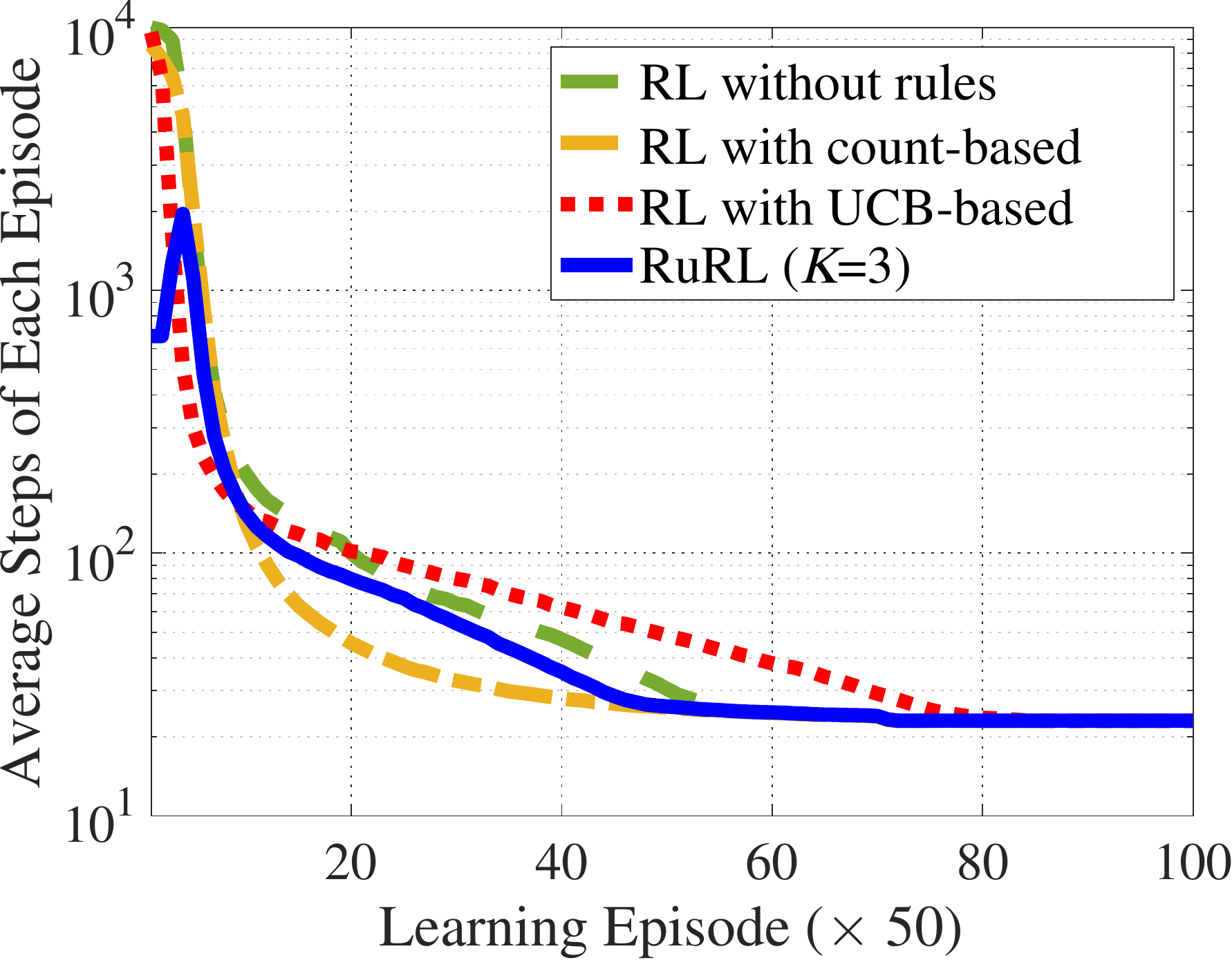}}
\caption{The average steps per episode of RL without rules, RL with count-based, RL with UCB-based and RuRL in single-room navigation tasks.}
 \label{fig:result1}
 \end{figure}

\begin{table}[tb]
	\centering
	\caption{Numerical results in terms of total learning steps of all tested algorithms in single-room tasks.}
	\begin{tabular}{l|rr}
		\toprule
		\multirow{2}{*}{Implementation algorithm} & \multicolumn{2}{c}{Q-learning ($\varepsilon$-greedy)}   \\
		& \tabincell{c}{Env 1}  
		&\tabincell{c}{Env 2}\\
		\midrule
		RL without Rules ({$\times 10^{5}$})  &18.18 & 19.06   \\
		RL with count-based ({$\times 10^{5}$})   & 14.90 & 15.60  \\
		RL with UCB-based ({$\times 10^{5}$})   & 11.27 & 11.08   \\
		RuRL ($K=3$) ({$\times 10^{5}$})     & 8.82 & 7.14   \\
		\bottomrule
	\end{tabular}
	\label{Table2}
\end{table}

First, all tested methods are implemented by the Q-learning algorithm.
The $\varepsilon$-greedy exploration strategy is used for RuRL, RL without rules, and RL with count-based.
As shown in Fig.~\ref{fig:result1}(a), all methods can find the optimal path with 23 steps in Env~$1$.
We observe that RL with count-based, RL with UCB-based and RuRL methods perform well compared with RL without rules.
The RL with count-based method demonstrates that internal reward signals positively affect the middle learning stage.
Compared to the RL without rules, RL with UCB-based approach obtains slightly superior performance with reduced regrets for efficient exploration throughout the learning process.
In contrast, RuRL obtains the best performance by improving the exploration efficiency which may benefit from the Pledge rule, where the efficiency of the rule for reducing the redundant space is minor.
Table~\ref{Table2} shows that the steps are already reduced by 51.46\% when $K = 3$ compared to RL without rules.
Under the situation that they all find the optimal path, A* and ACO algorithms need to switch directions for 11 and 13 times, respectively, while RuRL only requires 3 times.\footnote{See Appendix A in Supplementary Materials for details.}
Frequently switching directions will slow down the speed of mobile robots with more energy consumed and may be critical for the performance of the robot motion kinematics~\citep{ravankar2018path}.

Next, we test RuRL in Env 2 with obstacles, as shown in Fig.~\ref{fig:previous}(c).
From Fig.~\ref{fig:result1}(b), we can find that all tested algorithms can obtain the optimal path with 23 steps.
We also find that RuRL has greater performance improvement than that in Env 1, but the performance of RL with count-based and RL with UCB-based methods is slightly improved.
When there are more obstacles, the rules of reducing the redundant space play a more active role in enhancing exploration efficiency.
It can be obtained from Table~\ref{Table2} that the learning steps of RuRL are nearly 62.52\% less than that of RL without rules.
In general, RuRL can accelerate the learning process to a large extent, especially in the environments with obstacles.
Furthermore, the planned paths of A*, ACO, and RuRL methods need to switch directions for 11, 9, and 7 times, respectively.\footnotemark[4]

\subsection{Tasks in Multi-room Environments}\label{Sec4.3}
To further test the performance of RuRL, we use a larger and more complex navigation environment composed of multiple rooms, as shown in Fig.~\ref{fig:env3}.
The starting point and target point of Env 3 are set as $(85, 56)$ and $(3, 22)$.
The maximal numbers of learning episodes and learning steps per episode are set as 15000 and 20000, respectively. 
The bonus coefficient of count-based exploration strategy $\beta$ is set as 0.5.
Other hyperparameters corresponding to the exploration strategy and the Pledge rule can be found in Table~\ref{Table3}.

\begin{figure}[tb]
	\centering
	\includegraphics[height=3.7cm,width=6cm]{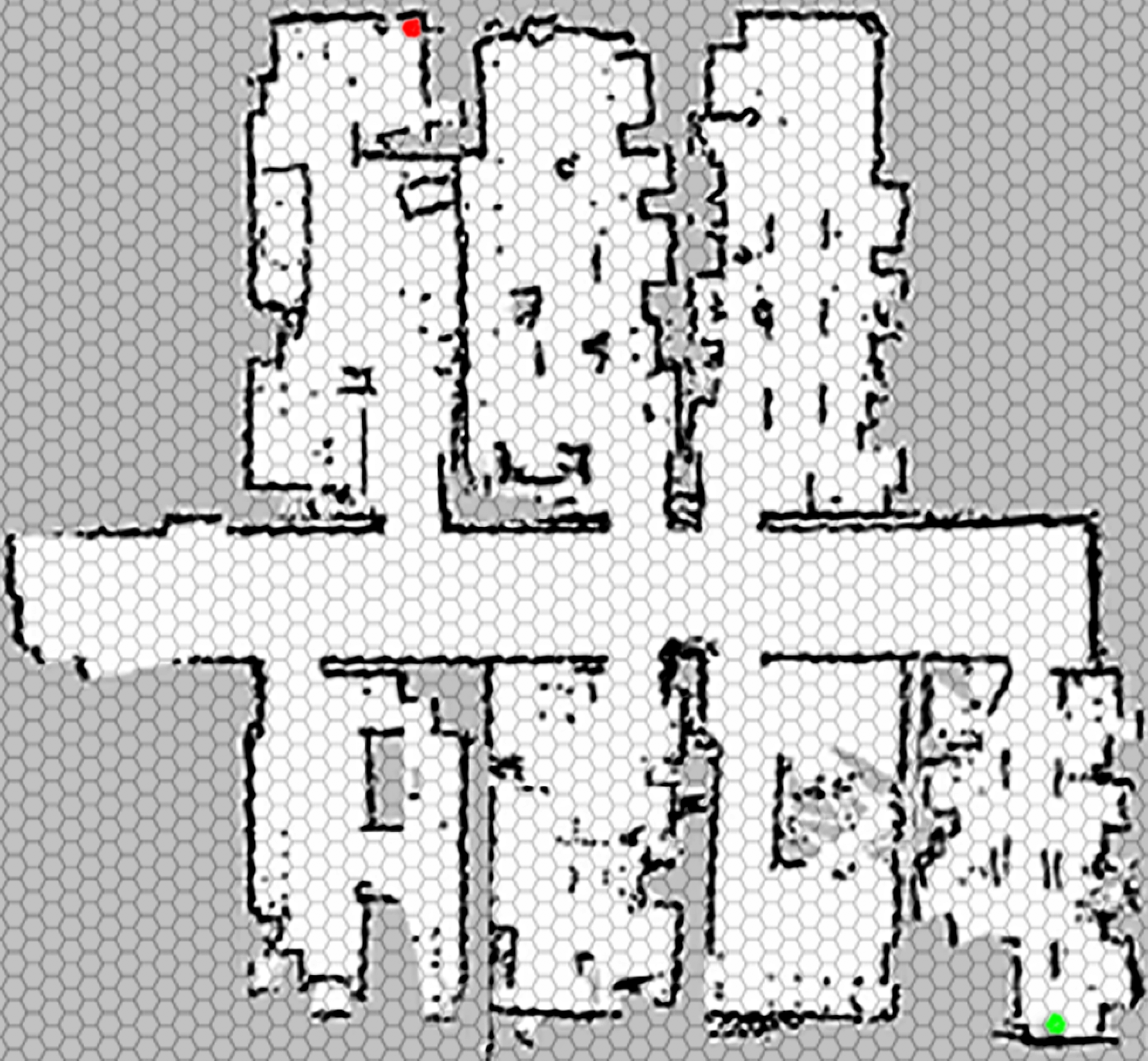}
	\caption{Env 3: the multi-room navigation environment (1640cm $\ast$ 1960cm) rasterizated into an 87$\ast$59 hexagonal grid map using Eq.~(\ref{equ:col}).}
	\label{fig:env3}
\end{figure}

\begin{table*}
\centering
\caption{Hyper-parameters in Env 3.}
\setlength{\tabcolsep}{0.26mm}
\begin{tabular}{c|cccc}
\toprule
\multirow{1}{*}{Implementation algorithm} & \multicolumn{1}{c}{Q-learning ($\varepsilon$-greedy)}& \multicolumn{1}{c} {Q-learning (Softmax)}& \multicolumn{1}{c} {SARSA ($\varepsilon$-greedy)}& \multicolumn{1}{c} {SARSA (Softmax)} \\

  \midrule
  Exploration ratio
   & { \(\left\{\begin{array}{l}{e^{-0.001 * \text {$\eta$}}(\text {$\eta$}<5000)} \\ {0}\end{array}\right.\)}
   &  {\(\left\{\begin{array}{l}{\frac{35}{\text {$\eta$} * 0.011+1}(\text { $\eta$}<3000)} \\ {1}\end{array}\right.\)} 
      &  {\(\left\{\begin{array}{l}{\frac{1}{\text {$\eta$} * 0.4}(\text { $\eta$}<500)} \\ {0}\end{array}\right.\)} 
   & {\(\left\{\begin{array}{l}{\frac{35}{\text {$\eta$} * 0.011+1}(\text { $\eta$}<3000)} \\ {1}\end{array}\right.\)} \\
  \# of episodes with Pledge rule $N$&700 &1000 &500&1000 \\
  Threshold steps of each episode $E$ & {\(\frac{20000}{0.1 *  \eta+8}\)}  
  & {\(\frac{20000}{0.1* \eta+10}\)} 
  &{\(\frac{20000}{0.1 * \eta+10}\)} 
  & {\(\frac{20000}{0.1* \eta+10}\)} \\
  \bottomrule
  \end{tabular}
  \label{Table3}
\end{table*}

In the multi-room task, the A* and ACO approaches obtain sub-optimal paths with 73 and 74 steps and need to switch directions for 29 and 27 times, respectively.
In contrast, RuRL requires only 14 times of switching directions while learning the optimal path with 72 steps, where RuRL is implemented by the Q-learning algorithm with $\varepsilon$-greedy exploration strategy.
Compared with A* and ACO algorithms, RuRL can find smoother paths with a potentially better optimality guarantee.\footnotemark[4]

Further, Fig.~\ref{fig:result3} presents the learning steps of all tested algorithms implemented by Q-learning and SARSA with the $\varepsilon$-greedy and Softmax strategies, and Table~\ref{Table4} shows corresponding numerical results.
It is clear that the number of learning steps of RL with count-based, RL with UCB-based and RuRL methods is lower than that of RL without rules. 
The RL with count-based method improves the performance with the $\varepsilon$-greedy strategy, and obtains sub-optimal policy with 74 steps due to paying more attention to visited states. 
However, it obtains reduced performance improvement with the Softmax strategy. 
Compared to the single-room tasks, the improvement of RL with UCB-based is reduced to $1.18\%$ since the regret scales linearly in the dimension of the state space.
In contrast, RuRL enables the agent to make quicker progress than the others on finding the optimal policy.
Taking the implementation of Q-learning with $\varepsilon$-greedy strategy as an example, the total learning steps are reduced by 71.79\% when using the proposed rules with $K = 3$.
Consistent with observations in Section \ref{Sec4.2}, RuRL better improves the learning performance in multi-room tasks, which is supposed to benefit from the distinct space reduction and the Pledge rule for finding the goal with fewer steps.

\begin{table*}
	\centering
	\caption{Numerical results in terms of total learning steps of all tested algorithms in multi-room tasks (Env 3).}
	\renewcommand\arraystretch{1.0}
	\begin{tabular}{l|ccccc}
		\toprule
		\multirow{2}{*}{Implementation algorithm} & RL without rules & RL with count-based & RL with UCB-based & RuRL~($K=3$) & Maximal reduction \\
		& ({$\times 10^{7}$}) & ({$\times 10^{7}$}) & ({$\times 10^{7}$}) & ({$\times 10^{6}$}) & (\%) \\
		\midrule
		Q-learning ($\varepsilon$-greedy)  &1.69 & 1.17  & \multirow{2}{*}{1.67}   & 4.77   & \bf{71.79\%}\\
		
		Q-learning (Softmax)                      &2.36 & 2.51 &  & 7.08 & \bf{69.93\%}\\
		
		SARSA ($\varepsilon$-greedy)      &1.23 & 0.79  & \multirow{2}{*}{1.24}   & 4.72  & \bf{61.69\%}\\
		
		SARSA (Softmax)                          &2.78 & 1.99  &   & 7.58 & \bf{72.78\%}\\
		\bottomrule
	\end{tabular}
	\label{Table4}
\end{table*}

In addition, we analyze the learning performance of RuRL as the optimization step $K$ increases.
We also adopt a statistical approach to analyze the relationship between the number of episodes $N$ with the Pledge rule employed and the performance of RuRL (see Appendix B and Appendix C in Supplementary Materials for details).

\begin{figure}[tb]
	\centering
	\subfigure[Q-learning with $\varepsilon$-greedy]{\includegraphics[height=3cm]{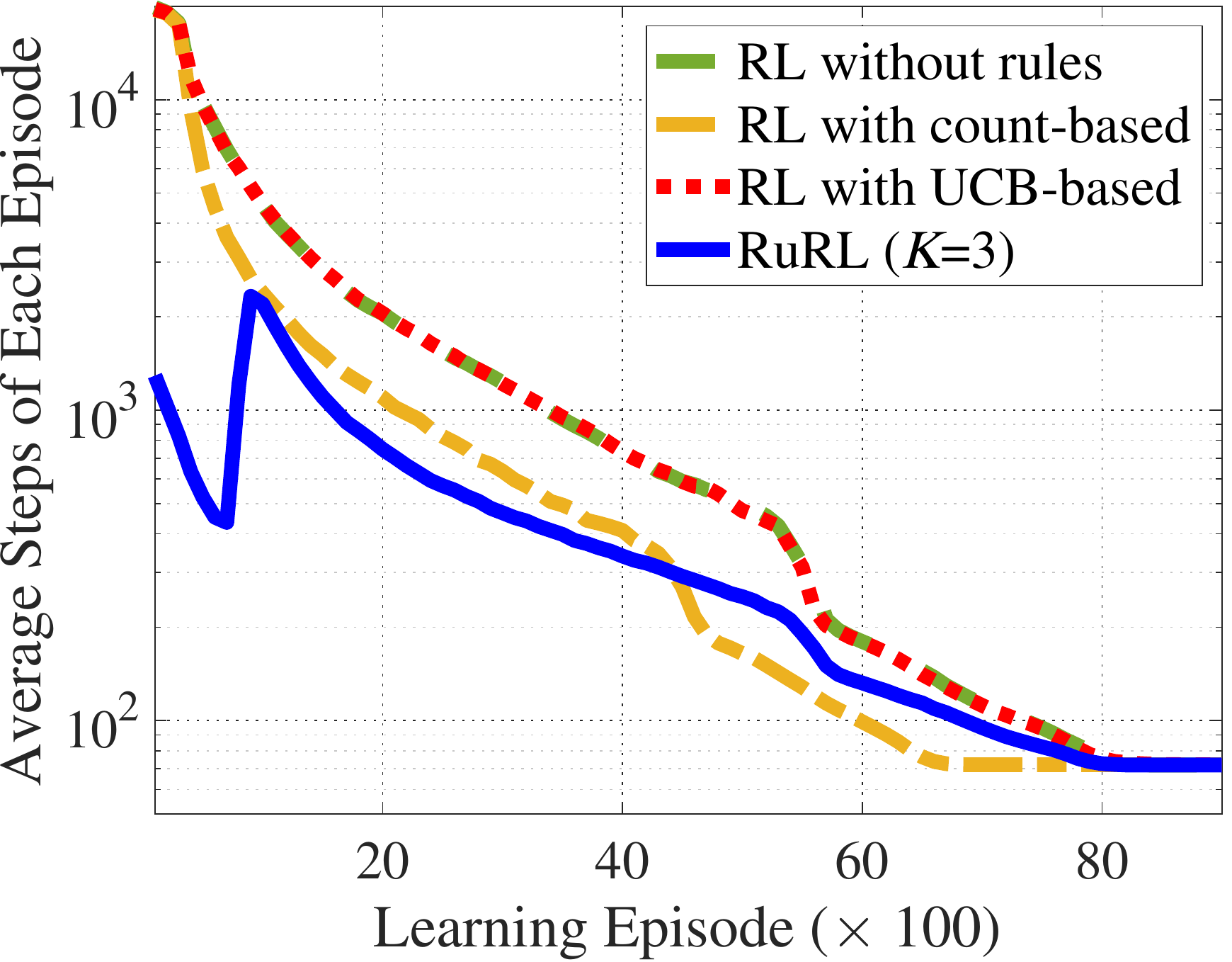}}\hspace{1em}
	\subfigure[Q-learning with Softmax]{\includegraphics[height=3cm]{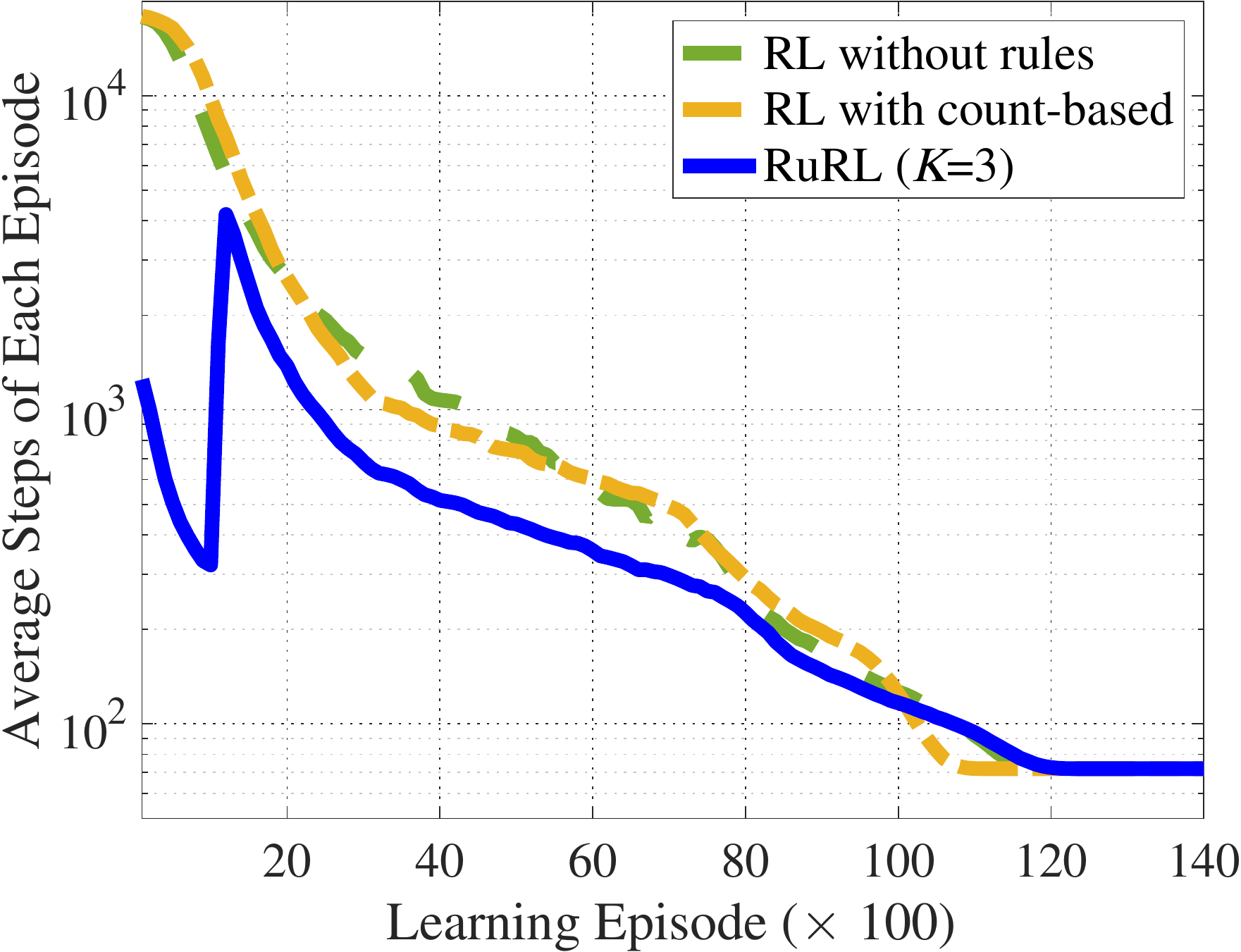}}     
	\subfigure[SARSA with $\varepsilon$-greedy]{\includegraphics[height=3cm]{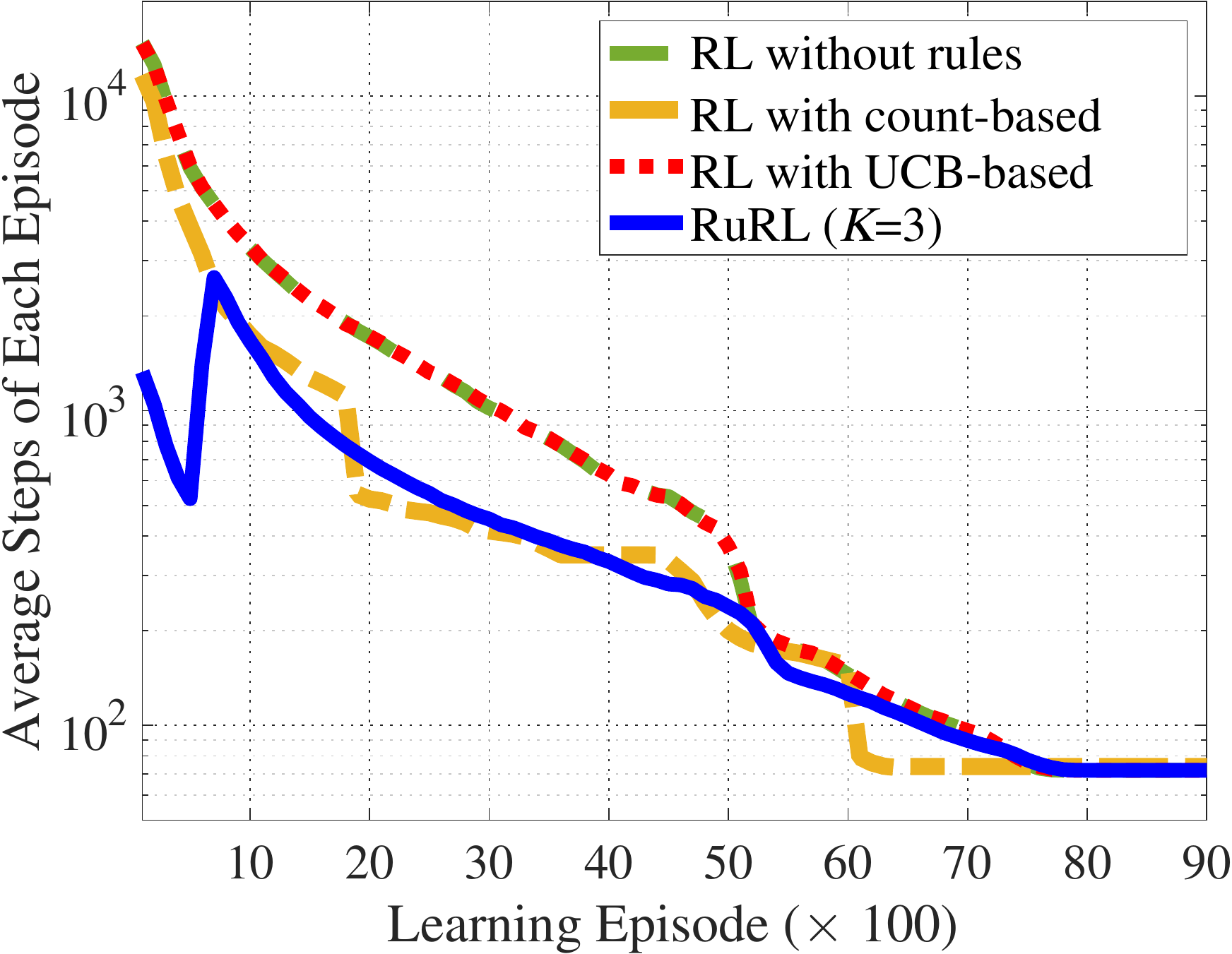}}\hspace{1em}
	\subfigure[SARSA with Softmax]{\includegraphics[height=3cm]{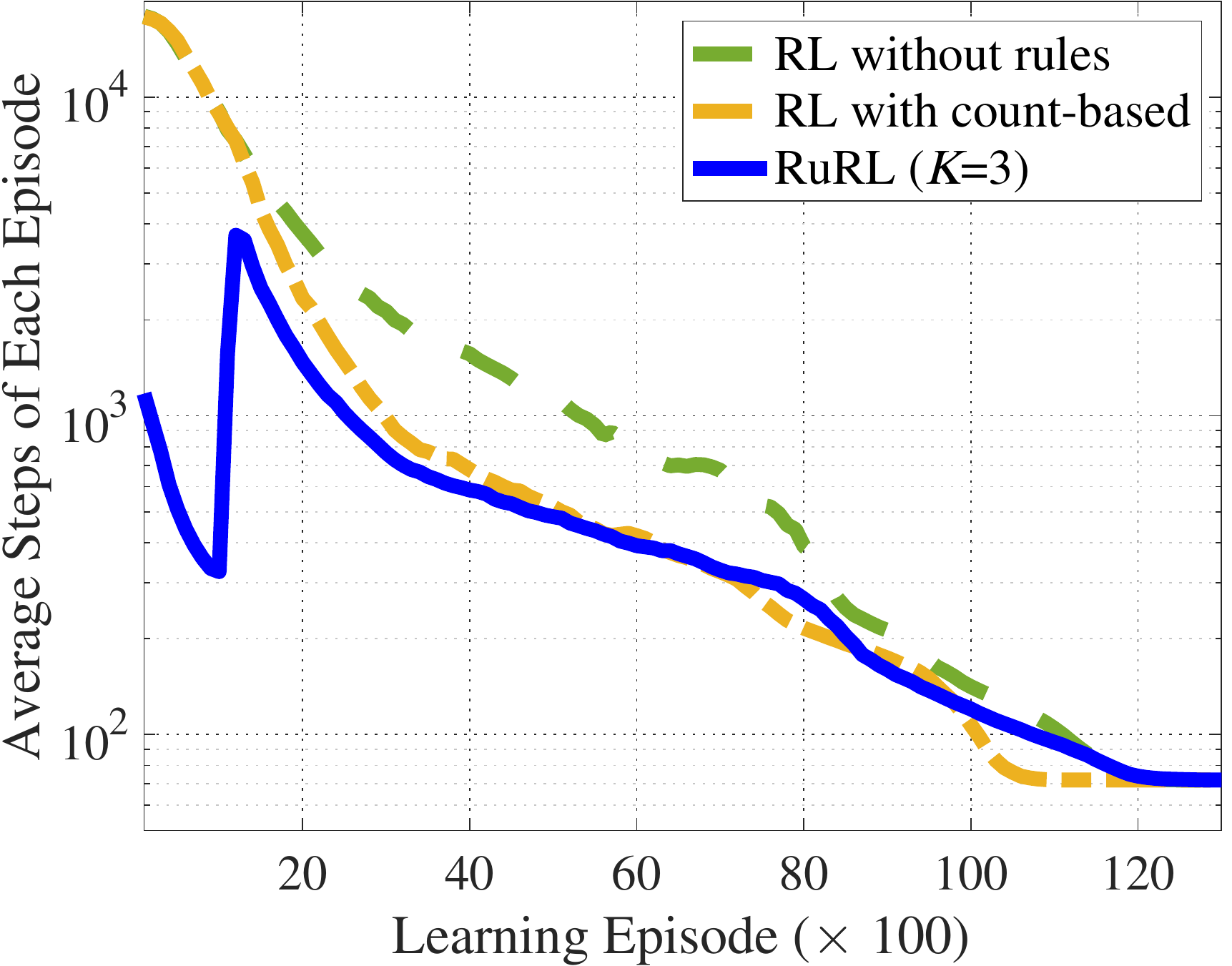}}
	\caption{The performance of all tested methods implemented by Q-learning and SARSA in the multi-room environment. The $\varepsilon$-greedy and Softmax strategies are used for RuRL, RL with count-based, and RL without rules.}
	\label{fig:result3}
\end{figure}

\section{Conclusion}\label{Sec5}
In this paper, we propose a rule-based RL (RuRL) algorithm for efficient robot navigation with space reduction, where three rules are applied to reduce the redundant exploration space and guide the exploration strategy.
Then, we evaluate RuRL on the single-room environments and a multi-room environment, where the maps are built using a SLAM mobile robot.
Experimental results demonstrate that RuRL can efficiently improve the navigation performance with good scalability.
Our future work will focus on more practical rules for advanced RL methods in the field of complex robotic control.

\footnotesize
\bibliography{reference}
\bibliographystyle{IEEEtranN}

\begin{IEEEbiography}[{\includegraphics[width=1in,height=1.25in,clip,keepaspectratio]{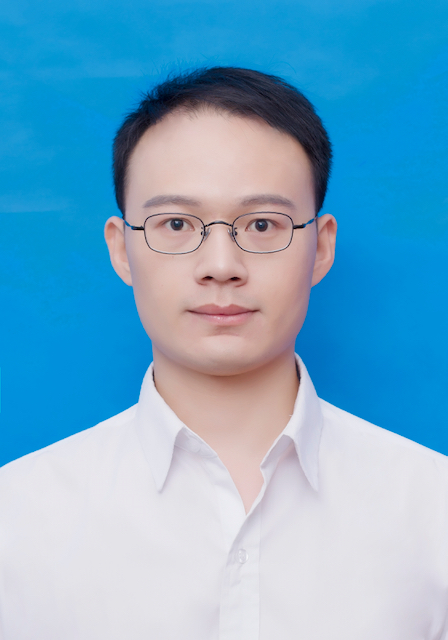}}]{Yuanyang Zhu}
	 received the B.E. degree in automation from the Department of Automation, Huaiyin Institute of Technology, Huai'an, China, in 2017, and the M.S. degree in the Department of Control and Systems Engineering, the School of Management and Engineering, Nanjing University, Nanjing, China, in 2020, where he is currently pursuing the Ph.D. degree.
	 His current research interests include reinforcement learning, machine learning, and robotics.
\end{IEEEbiography}

\begin{IEEEbiography}[{\includegraphics[width=1in,height=1.25in,clip,keepaspectratio]{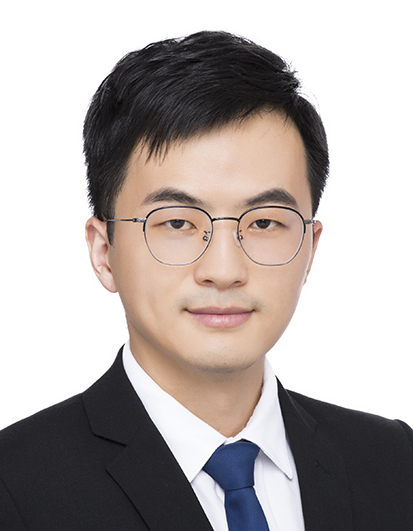}}]{Zhi Wang}
	(S'19-M'20) received the Ph.D. degree in machine learning from the Department of Systems Engineering and Engineering Management, City University of Hong Kong, Hong Kong, China, in 2019, and the B.E. degree in automation from Nanjing University, Nanjing, China, in 2015.
	He is currently an Assistant Professor in the Department of Control and Systems Engineering, Nanjing University.
	His current research interests include reinforcement learning, machine learning, and robotics.
\end{IEEEbiography}

\begin{IEEEbiography}[{\includegraphics[width=1.0in,height=1.25in,clip,keepaspectratio]{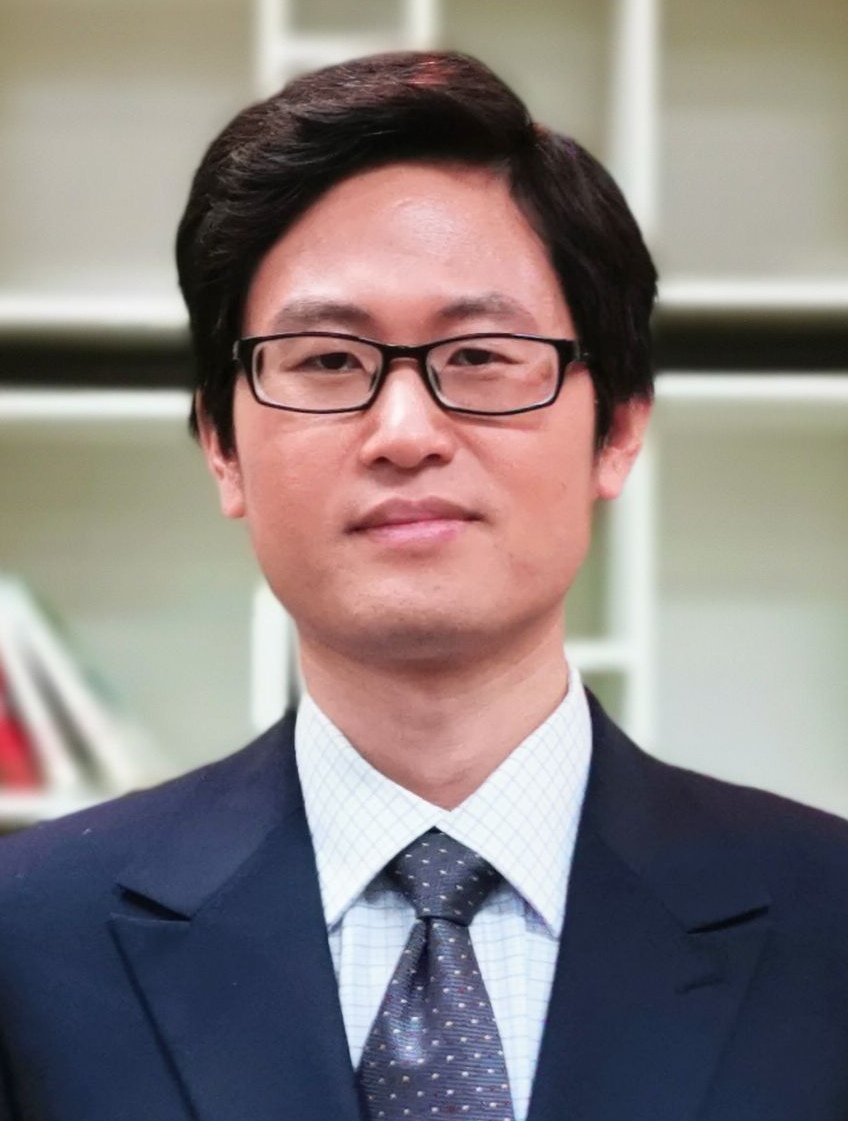}}]{Chunlin Chen}
	(S'05-M'06-SM'21) received the B.E. degree in automatic control and Ph.D. degree in control science and engineering from the University of Science and Technology of China, Hefei, China, in 2001 and 2006, respectively.
	He is currently a professor and the chair of the Department of Control and Systems Engineering, Nanjing University, Nanjing University.
	
	His current research interests include machine learning, intelligent control and quantum control.
	He is Chair of Technical Committee on Quantum Cybernetics, IEEE Systems, Man and Cybernetics Society.
\end{IEEEbiography}

\begin{IEEEbiography}[{\includegraphics[width=1.0in,height=1.25in,clip,keepaspectratio]{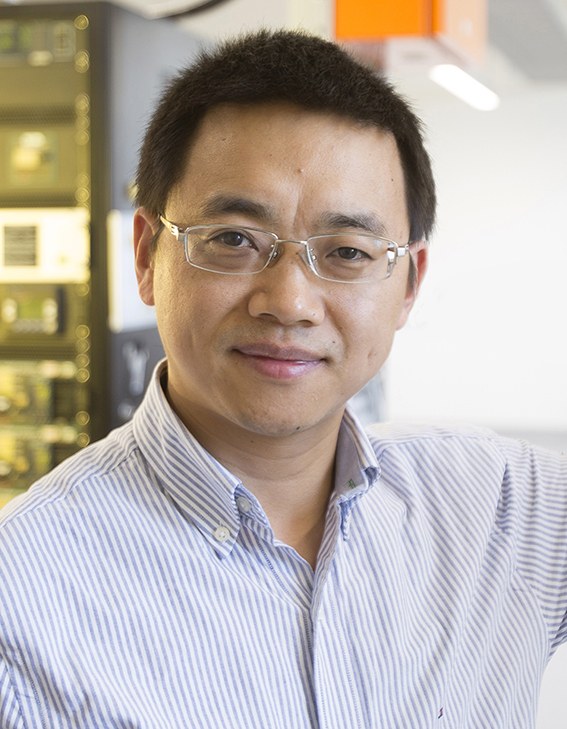}}]{Daoyi Dong}
	(S'05-M'06-SM'11) is currently a Scientia Associate Professor  at the University of New South Wales, Canberra, Australia. 
	He received a B.E. degree and a Ph.D. degree in engineering from the University of Science and Technology of China, Hefei, China, in 2001 and 2006, respectively. 
	
	His research interests include machine learning and quantum cybernetics. He was awarded an ACA Temasek Young Educator Award by The Asian Control Association and is a recipient of an International Collaboration Award, Discovery International Award and an Australian Post-Doctoral Fellowship from the Australian Research Council, and Humboldt Research Fellowship from Alexander von Humboldt Foundation in Germany.
	He serves as an Associate Editor of IEEE Transactions on Neural Networks and Learning Systems, and a Technical Editor of IEEE/ASME Transactions on Mechatronics.	
\end{IEEEbiography}

\clearpage 


\section*{Supplementary material}

# Supplementary Materials: Rule-Based Reinforcement Learning for Efficient Robot Navigation with Space Reduction

Yuanyang Zhu, Zhi Wang, Chunlin Chen, Daoyi Dong

## Appendix A: Supplementary experimental results from the final path in Env 1, 2 and 3.

This appendix shows the full final planned paths of A*, ACO and RuRL methods, which are listed in Figs. S1(a)-(c).

As shown in Fig. S1(a), A*, ACO and RuRL methods can find the optimal path with 23 steps in Env 1. When executing the optimal path, A* and ACO algorithms need to switch directions for 11 and 13 times, respectively, while RuRL only requires 3 times. Frequently switching directions will slow down the speed of mobile robots with more consumed energy and may be critical of the performance of the robot motion kinematics [46].

From Fig. S1(b), we find that A*, ACO and RuRL can obtain the optimal path with 23 steps. Furthermore, the planned path of A*, ACO and RuRL methods need to switch directions for 11, 9 and 9 times, respectively.

In the multi-room environment, the A* and ACO approaches obtain sub-optimal paths with 73 and 74 steps and need to switch directions for 29 and 27 times, respectively, as shown in Fig. S1-(c). In contrast, RuRL requires only 14 times of switching directions while learning the optimal path with 72 steps, where RuRL is implemented by the Q-learning algorithm with $\varepsilon$-greedy exploration strategy. Compared with A* and ACO classic algorithms, we can find that RuRL is capable of finding smoother paths with a better optimality guarantee.

## Appendix B: Supplementary experimental results with different optimization step $K$.

This appendix shows the improved learning performance with the optimization step $K$ increasing.

From Table SI, as the optimization step $K$ increases, the learning performance is improved slightly. On the other hand, a larger optimization step will lead to large computational cost. Generally, a moderate optimization step (e.g., $K = 2$) is enough for an appropriate trade-off between the performance improvement and computational cost. As shown in Fig. S2, when the Pledge rule is withdrawn, the exploratory strategy requires more exploration steps to ensure optimal path and the number of learning steps will increase suddenly, while the steps of RuRL are still much smaller than those of the other methods. It indicates that the rule can effectively guide the early exploration strategy in the reduced space.

Table. S I: Numerical results in terms of total learning steps of all tested algorithms in multi-room tasks.

| Implementation algorithm | Env 3 | | | | |
|---|---|---|---|---|---|
| | RL without Rules ($\times 10^7$) | RuRL ($K = 1$) ($\times 10^6$) | RuRL ($K = 2$) ($\times 10^6$) | RuRL ($K = 3$) ($\times 10^6$) | Maximal reduction (%) |
| Q-learning ($\varepsilon$-greedy) | 1.69 | 5.18 | 4.96 | 4.77 | **71.79%** |
| Q-learning (Softmax) | 2.36 | 7.86 | 7.29 | 7.08 | **69.93%** |
| SARSA ($\varepsilon$-greedy) | 1.23 | 5.12 | 4.85 | 4.72 | **61.69%** |
| SARSA (Softmax) | 2.78 | 8.45 | 7.89 | 7.58 | **72.78%** |

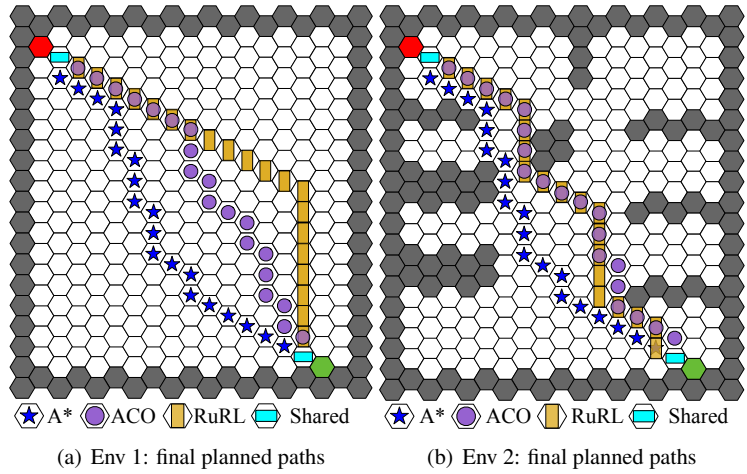

(a) Env 1: final planned paths    (b) Env 2: final planned paths

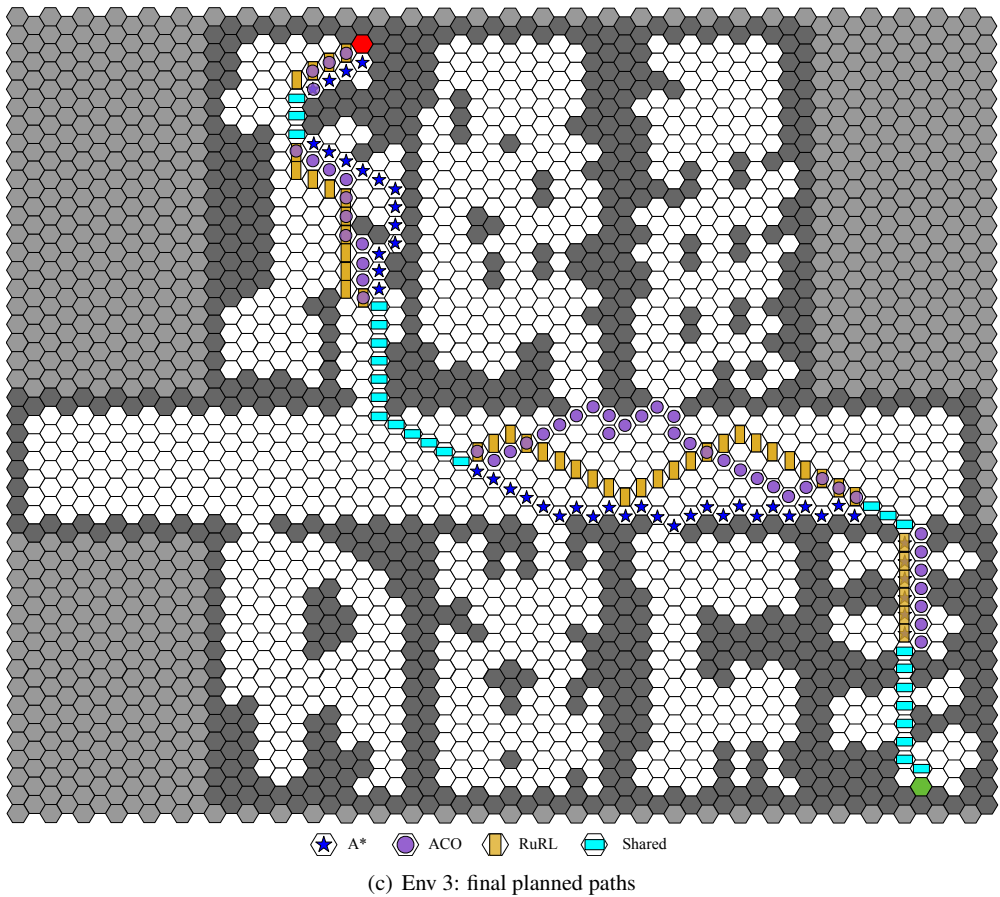

(c) Env 3: final planned paths

Fig. S 1: The paths planned by A*, ACO and RuRL are denoted with blue stars, purple circles and yellow vertical bars, respectively. The grids marked by the dark blue horizontal bars represent the shared part of their paths.

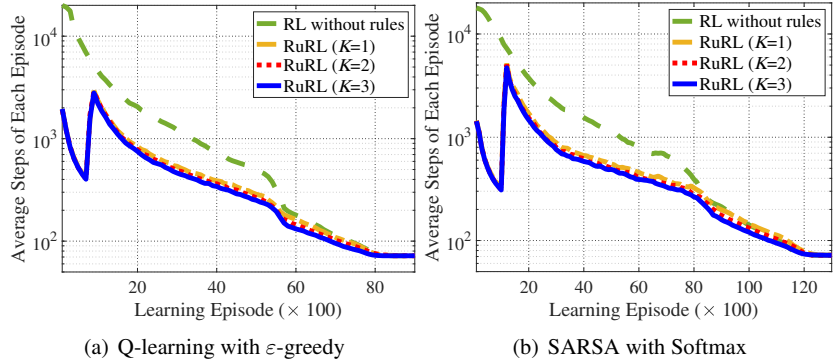

Fig. S 2: The average steps per episode of RL and RuRL with $K = 1, 2, 3$ in multi-room navigation tasks.

## Appendix C: Supplementary experimental results with different numbers of episodes $N$ with the Pledge rule.

This appendix shows the full results of the performance of RuRL with different numbers of episodes $N$ with the Pledge rule, which are listed in Fig. S3 and Table SII, respectively.

In addition, we analyze the relationship between the number of episodes $N$ with the Pledge rule and the performance of RuRL. We set different values of $N$ with optimization step $K = 3$ for tasks in the single-room environment with obstacles and the multi-room environment, respectively. The learning curves and corresponding numerical results are presented in Fig. S3 and Table SII, respectively. It can be observed that the learning steps decrease as the Pledge rule guides the early exploration strategy for more episodes. A potential drawback is that a too large $N$ might lead to sub-optimal policies. Generally, using the Pledge rule for an appropriate number of episodes can accelerate the navigation process at the beginning of the learning process, without influencing the final optimal policies.

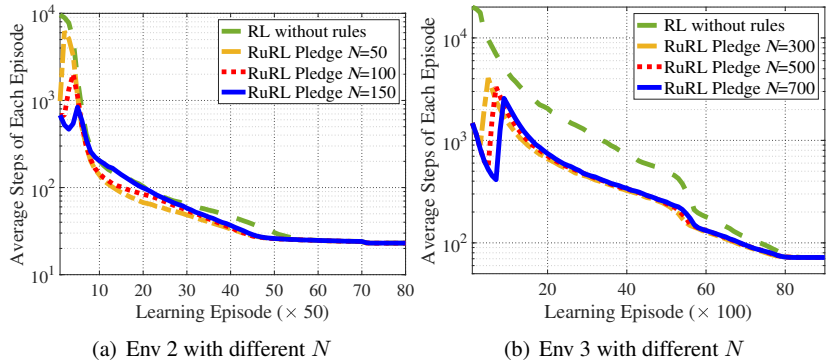

Fig. S 3: The performance of RuRL using different $N$ episodes with the Pledge rule.

Table. S II: Numerical results in terms of total learning steps of RL and RuRL, which are implemented by Q-learning with different N episodes using the Pledge rule under the $\varepsilon$-greedy exploration strategy in Env 2 and Env 3.

| Algorithm Name | Env 2 | | | | | Env 3 | | | | |
|---|---|---|---|---|---|---|---|---|---|---|
| | RL without Rules ($\times 10^6$) | RuRL $N=100$ ($\times 10^6$) | RuRL $N=150$ ($\times 10^5$) | RuRL $N=200$ ($\times 10^5$) | Maximal Reduction (%) | RL without Rules ($\times 10^7$) | RuRL $N=100$ ($\times 10^6$) | RuRL $N=300$ ($\times 10^6$) | RuRL $N=500$ ($\times 10^6$) | Maximal Reduction (%) |
| Q-learning($\varepsilon$-greedy) | 1.91 | 1.10 | 7.14 | 5.38 | **71.78%** | 1.69 | 5.48 | 5.13 | 4.77 | **71.79%** |

# Rule-Based Reinforcement Learning for Efficient Robot Navigation with Space Reduction



\end{document}